\documentclass{article}
\usepackage{arxivrelease}
\usepackage{times}

\usepackage[utf8]{inputenc} 
\usepackage[T1]{fontenc}    
\usepackage{hyperref}       
\usepackage{url}            
\usepackage{booktabs}       
\usepackage{amsfonts}       
\usepackage{nicefrac}       
\usepackage{microtype}      
\usepackage{xcolor}         

\usepackage[backend=biber, style=apa, doi=false,isbn=false,url=false,eprint=false]{biblatex}
\bibliography{aistats_main}

\usepackage{graphicx}

\usepackage{amsmath}
\usepackage{amssymb}
\usepackage{dsfont}
\usepackage{mathtools}
\usepackage[bb=boondox]{mathalpha}

\usepackage{hyperref} 
\usepackage{cleveref} 
\crefname{enumi}{}{} 
\usepackage{xr}

\usepackage{amsthm} 
\usepackage{thmtools} 
\usepackage{thm-restate}

\newtheorem{lemma}{Lemma}[section]

\newtheorem{definition}{Definition}[section]

\crefname{proposition}{proposition}{propositions}
\crefname{figure}{fig.}{fig.}
\crefname{equation}{eq.}{eq.}

\usepackage{enumitem}

\usepackage{verbatim}

\newcommand{\noisyy}{\tilde{y}}
\newcommand{\Noisyy}{\tilde{Y}}
\newcommand{\onehoty}{e^{(y)}}
\newcommand{\onehotY}{e^{(Y)}}

\newcommand{\model}{f}
\newcommand{\truemodel}{g}
\newcommand{\noisymodel}{\tilde{g}}

\newcommand{\prob}{\mathds{P}}

\newcommand{\yspace}{\mathcal{Y}}
\newcommand{\xspace}{\mathcal{X}}
\newcommand{\modelspace}[1]{\mathcal{F}_{\mathcal{#1}}}
\newcommand{\losspace}[1]{\mathcal{L}_{\text{#1}}}

\newcommand{\argmax}[1]{\underset{#1}{\text{argmax}}\;}
\newcommand{\argmin}[1]{\underset{#1}{\text{argmin}}\;}
\newcommand{\entropy}[1]{\mathcal{H}[#1]}
\newcommand{\expectation}{\mathds{E}}
\newcommand{\loss}{\ell}
\newcommand{\risk}{\mathcal{R}_{\loss}}
\newcommand{\noisyrisk}{\tilde{\mathcal{R}}_{\mathcal{\loss}}}

\newcommand{\sforall}{\; \forall}

\usepackage{tikz}
\usetikzlibrary{shapes.geometric,arrows}
\usepackage{pgfplots}
\usepgfplotslibrary{groupplots}

\usepackage{caption}
\usepackage{subcaption}

\newif\ifsubmission
\submissionfalse

\newif\ifarxiv
\arxivtrue

\newif\ifrenderTikz
\renderTikztrue

\let\oldtikzpicture\tikzpicture
\let\oldendtikzpicture\endtikzpicture

\renewenvironment{tikzpicture}{%
    \ifrenderTikz\expandafter\oldtikzpicture%
    \else\comment%
    \fi
}{%
    \ifrenderTikz\oldendtikzpicture%
    \else\endcomment%
    \fi
}

\title{Robustness and Reliability When Training With Noisy Labels}%

\author{
{\bf Amanda Olmin} \\
amanda.olmin@liu.se \\
Link\"oping University, Sweden\\
\And
{\bf Fredrik Lindsten} \\
fredrik.lindsten@liu.se \\
Link\"oping University, Sweden\\
}

\date{}

\begin{document}

\onecolumn{
	\noindent{}This is an edited version of the paper \textit{Robustness and Reliability When Training With Noisy Labels} published in \textit{Proceedings of The 25th International Conference on Artificial Intelligence and Statistics}, 2022.
	\\
	
	\noindent Please cite as:
	\small{
		\begin{verbatim}
			@InProceedings{Olmin2022,
				author = {{O}lmin, {A}manda and {L}indsten, {F}redrik},
				booktitle = {{I}nternational {C}onference on {A}rtificial {I}ntelligence and {S}tatistics},
				title = {{R}obustness and {R}eliability {W}hen {T}raining {W}ith {N}oisy {L}abels}},
			pages = 922-942,
			year = {2022}
		}
\end{verbatim}}

}
\newpage

\maketitle

\begin{refsection}
\begin{abstract}
 Labelling of data for supervised learning can be costly and time-consuming and the risk of incorporating label noise in large data sets is imminent. When training a flexible discriminative model using a strictly proper loss, such noise will inevitably shift the solution towards the conditional distribution over noisy labels. Nevertheless, while deep neural networks have proven capable of fitting random labels, regularisation and the use of robust loss functions empirically mitigate the effects of label noise. However, such observations concern robustness in accuracy, which is insufficient if reliable uncertainty quantification is critical. We demonstrate this by analysing the properties of the conditional distribution over noisy labels for an input-dependent noise model. In addition, we evaluate the set of robust loss functions characterised by noise-insensitive, asymptotic risk minimisers. We find that strictly proper and robust loss functions both offer asymptotic robustness in accuracy, but neither guarantee that the final model is calibrated. Moreover, even with robust loss functions, overfitting is an issue in practice. With these results, we aim to explain observed robustness of common training practices, such as early stopping, to label noise. In addition, we aim to encourage the development of new noise-robust algorithms that not only preserve accuracy but that also ensure reliability.

\end{abstract}

\ifarxiv
\section{Introduction and preview of contribution}
\else
\section{Introduction and\\ preview}
\fi
\label{sec:introduction}

Deep neural networks have been successfully applied in many fields. However, because of their high complexity, they typically require a large amount of data for training. The process of annotating a large and possibly high-dimensional data set is costly, time-consuming and risks incorporating label noise in the training data set. For difficult boundary cases, even expert annotators can disagree about the true label. Evidently, as supervised training of discriminative models is highly dependent on the existence of labelled data, label noise poses a risk of hurting model performance. Firstly, the noise can cause the asymptotic risk minima to shift, resulting in a model approximating the conditional distribution over noisy, instead of clean, labels. 
Secondly, a flexible model might overfit to the noise in the data. 

Central to the training of probabilistic predictive models are so called proper loss functions \parencite{Gneiting2007, Reid2010}. A loss function is proper if it, for class variable $Y$ and input $X$, is asymptotically minimised by the true conditional probability $\model^*(x) = \prob(Y \mid X=x)$. Moreover, if the (asymptotic) risk minimiser $\prob(Y\mid X)$ is unique, the loss is \emph{strictly proper}. Strictly proper loss functions, such as the commonly used categorical cross-entropy loss, encourage reliable uncertainty quantification during training. Using such a loss function, at least there is asymptotic, theoretical grounds for obtaining a model close to $\prob(Y  \mid X)$. A weaker, and perhaps more realistic, reliability condition than requiring that the model recovers the true class probability, is the notion of calibration. A model is calibrated if it reports class probabilities that agree with the observed prediction error frequency (see \cref{def:calibrated_def}). Note that $\prob(Y \mid X)$ is calibrated by definition.

In the presence of noisy labels, $\Noisyy$, the asymptotic risk minimiser of a strictly proper loss function is $\tilde \model^*(x) = \prob(\Noisyy  \mid X=x)$ instead of $\prob(Y \mid X=x)$. Hence, loss functions that are insensitive to noise have been proposed as alternatives, e.g. in \parencite{Ghosh2015, Ghosh2017, Charoenphakdee2019, Zhang2018, Wang2019}. Specifically, fully \emph{robust loss functions} are defined as having risk minimisers that are unaffected by label noise under certain assumptions on the noise distribution \parencite{Ghosh2015, Ghosh2017}. 
Although such a condition is sufficient for achieving robustness in accuracy, 
it does not imply that the risk minimisers are reliable. This is a conceivable issue since uncertainty quantification is critical in many applications of machine learning. Intuitively, accurate reasoning about uncertainties becomes even more relevant in the case of noisy annotations. 

In practice, the risk minimiser of a loss is only part of the story. As modern neural networks are often of high capacity, they are capable of overfitting to training data, also when labels are afflicted by noise \parencite{Zhang2017}. Conceptually, training with any loss function corresponds to a, possibly implicit, assumption that the training trajectory will pass "close" to the risk minimiser $\model^*$, before drifting off into overfitting. We illustrate this in \cref{fig:intuition_sketch_a} where we think of the training dynamics as consisting of two phases: a convergent phase where the model approaches $\model^*$, followed by a divergent (overfitting) phase where the distance between the model and the risk minimiser increases. We will not elaborate on the details of the training dynamics, nor characterise what "close" means. Still, we argue that this assumption underlies common practices of training neural networks, such as using early stopping to halt the trajectory as close as possible to $\model^*$. It is further supported by observations that overparameterised models tend to learn general patterns in the training data before overfitting to noisy examples \parencite{Arpit2017}.
We will use this assumption of two phases of training to illustrate some of the key points in the paper.

\begin{figure*}[t]
    \begin{subfigure}{0.33\textwidth}
        \ifarxiv
\newcommand\figscale{0.55}
\else
\newcommand\figscale{0.55}
\fi

\begin{tikzpicture}[scale = \figscale]

\pgfmathsetseed{236}

\definecolor{color0}{rgb}{0.705673158,0.01555616,0.150232812}
\definecolor{color1}{rgb}{0.2298057,0.298717966,0.753683153}

\draw[black] plot[smooth cycle, thick, tension=.8]
  coordinates{(0,1)  (3,2.5) (6,2) (6.5,0) (4,-1.5) (1,-1.6)};

\filldraw[black] (1,0) circle (2pt) node[anchor=center] (f0) {\normalsize{$\model^{0}\;\;\;\;\;$}};
\filldraw[black] (2.7,0.3) circle (2pt) node[anchor=center] (fs) {};
\node[anchor=south] at (2.6, 0) (fst) {\normalsize{$\model^{*}\;$}};
\filldraw[black] (4,1.9) circle (2pt) node[anchor=center] (fb) {};

\node[thick,] at (7, -1) (space) {\large{$\mathcal{F}$}};


\node (start) at ($(f0)$) {};
\def\step{0.9}
\def\minstep{0.3}
\foreach \i in {1,2} { %
        \node (end) at ($(start)+({max(\minstep,\step*rnd)}, \step*rnd-\step*0.9)$) {};
        
        \ifodd\i 
        \draw[->, color1, thick] (start) -- (end.center);
        \else 
        \draw[->, color1, thick] (start) -- (end.center);
        \fi
        
        \node (start) at (end.center) {};
        \pgfmathparse{\step-0.05}
        \xdef\step{\pgfmathresult}
}

\node (fsa) at ($(fs) + (0.1, -0.15)$) {};
\draw[->, color1, thick] (end) -- (fsa.center);

\node (start) at ($(fsa)$) {};
\foreach \i in {1,...,3} {
        \node (end) at ($(start)+({max(\minstep,\step*rnd)}, \step*rnd*2.5)$) {};
        \draw[->, color1, thick] (start) -- (end.center);
        
        \node (start) at (end) {};
        \pgfmathparse{\step-0.1}
        \xdef\step{\pgfmathresult}
}
\draw[->, color1, thick] (end) -- (fb.center);

\end{tikzpicture}
        \caption{}
        \label{fig:intuition_sketch_a}
    \end{subfigure}%
    \begin{subfigure}{0.33\textwidth}
        \ifarxiv
\newcommand\figscale{0.55}
\else
\newcommand\figscale{0.55}
\fi

\begin{tikzpicture}[scale = \figscale]
\pgfmathsetseed{236}

\definecolor{color0}{rgb}{0.89904617,0.439559467,0.343229596}
\definecolor{color1}{rgb}{0.2298057,0.298717966,0.753683153}

\draw[black] plot[smooth cycle, thick, tension=.8]
  coordinates{(0,1)  (3,2.5) (6,2) (6.5,0) (4,-1.5) (1,-1.6)};

\filldraw[black] (1,0) circle (2pt) node[anchor=center] (f0) {\normalsize{$\model^{0}\;\;\;\;\;$}};
\filldraw[black] (2.7,0.3) circle (2pt) node[anchor=center] (fs) {};
\node[anchor=south] at (2.6, 0) (fst) {\normalsize{$\model^{*}\;$}};
\filldraw[black] (4,1.9) circle (2pt) node[anchor=center] (fb) {\normalsize{}}; 

\filldraw[black] (2.7,-0.4) circle (2pt) node[anchor=center] (fsn) {};
\node[anchor=north] at (2.5,-0.3) (fsnt) {\normalsize{$\tilde{\model}^{*}$}};
\filldraw[black] (4.8,-0.2) circle (2pt) node[anchor=center] (fbn) {};
\node[anchor=west] at (4.7,-0.2) (fbnt) {\normalsize{}}; 

\node[thick,] at (7, -1) (space) {\large{$\mathcal{F}$}};


\node (start) at ($(f0)$) {};
\def\step{0.9}
\def\minstep{0.3}
\foreach \i in {1,2} { %
        \node (end) at ($(start)+({max(\minstep,\step*rnd)}, \step*rnd-\step*0.9)$) {};
        
        \draw[->, color1, thick] (start) -- (end.center);
        
        \node (start) at (end.center) {};
        \pgfmathparse{\step-0.05}
        \xdef\step{\pgfmathresult}
}

\node (fsa) at ($(fs) + (0.1, -0.15)$) {};
\draw[->, color1, thick] (end) -- (fsa.center);

\node (start) at ($(fsa)$) {};
\foreach \i in {1,...,3} {
        \node (end) at ($(start)+({max(\minstep,\step*rnd)}, \step*rnd*2.5)$) {};
        \draw[->, color1, thick] (start) -- (end.center);
        
        \node (start) at (end) {};
        \pgfmathparse{\step-0.1}
        \xdef\step{\pgfmathresult}
}
\draw[->, color1, thick] (end) -- (fb.center);

\def\step{0.9}
\node (start) at ($(f0)$) {};
\foreach \i in {1,...,2} {
        \node (end) at ($(start)+({max(\minstep+0.25,\step*rnd)}, \step*rnd-\step*0.5)$) {};
        
        \draw[->, color0, thick] (start) -- (end.center);
        
        \node (start) at (end) {};
        \pgfmathparse{\step-0.1}
        \xdef\step{\pgfmathresult}
}

\node (fsna) at ($(fsn) + (0, 0.2)$) {};
\draw[->, color0, thick] (end) -- (fsna.center);

\node (start) at ($(fsna)$) {};
\foreach \i in {1,...,3} {
        \node (end) at ($(start)+({max(\minstep+0.2,\step*rnd)}, \step*rnd - 0.7*\step)$) {};
        
        \draw[->, color0, thick] (start) -- (end.center);
        
        \node (start) at (end) {};
        \pgfmathparse{\step-0.1}
        \xdef\step{\pgfmathresult}
}
\draw[->, color0, thick] (end) -- (fbn.center);

\end{tikzpicture}
        \caption{}
        \label{fig:intuition_sketch_b}
    \end{subfigure}%
    \begin{subfigure}{0.33\textwidth}
        \ifarxiv
\newcommand\figscale{0.55}
\else
\newcommand\figscale{0.55}
\fi

\begin{tikzpicture}[scale = \figscale]

\pgfmathsetseed{236}

\definecolor{color0}{rgb}{0.89904617,0.439559467,0.343229596}
\definecolor{color1}{rgb}{0.2298057,0.298717966,0.753683153}

\draw[black] plot[smooth cycle, thick, tension=.8]
  coordinates{(0,1)  (3,2.5) (6,2) (6.5,0) (4,-1.5) (1,-1.6)};

\filldraw[black] (1,0) circle (2pt) node[anchor=center] (f0) {\normalsize{$\model^{0}\;\;\;\;\;$}};
\filldraw[black] (2.7,0.3) circle (2pt) node[anchor=center] (fs) {};
\node[anchor=south] at (2.6, 0) (fst) {\normalsize{$\model^{*}\;$}};

\filldraw[black] (4,1.9) circle (2pt) node[anchor=center] (fb) {}; 

\filldraw[black] (4.8,-0.2) circle (2pt) node[anchor=center] (fbn) {};
\node[anchor=west] at (4.7,-0.2) (fbnt) {}; 

\node[thick,] at (7, -1) (space) {\large{$\mathcal{F}$}};


\node (start) at ($(f0)$) {};
\def\step{0.9}
\def\minstep{0.3}
\foreach \i in {1,2} { %
        \node (end) at ($(start)+({max(\minstep,\step*rnd)}, \step*rnd-\step*0.9)$) {};
        
        \draw[->, color1, thick] (start) -- (end.center);
        
        \node (start) at (end.center) {};
        \pgfmathparse{\step-0.05}
        \xdef\step{\pgfmathresult}
}

\node (fsa) at ($(fs) + (0.1, -0.15)$) {};
\draw[->, color1, thick] (end) -- (fsa.center);

\node (start) at ($(fsa)$) {};
\foreach \i in {1,...,3} {
        \node (end) at ($(start)+({max(\minstep,\step*rnd)}, \step*rnd*2.5)$) {};
        \draw[->, color1, thick] (start) -- (end.center);
        
        \node (start) at (end) {};
        \pgfmathparse{\step-0.1}
        \xdef\step{\pgfmathresult}
}
\draw[->, color1, thick] (end) -- (fb.center);

\def\step{0.9}
\node (start) at ($(f0)$) {};
\foreach \i in {1,...,2} {
        \node (end) at ($(start)+({max(\minstep+0.25,\step*rnd)}, \step*rnd-\step*0.4)$) {};
        
        \draw[->, color0, thick] (start) -- (end.center);
        
        \node (start) at (end) {};
        \pgfmathparse{\step-0.08}
        \xdef\step{\pgfmathresult}
}

\node (fsna) at ($(fs) + (0, -0.15)$) {};
\draw[->, color0, thick] (end) -- (fsna.center);

\node (start) at ($(fsna)$) {};
\foreach \i in {1,...,3} {
        \node (end) at ($(start)+({max(\minstep+0.2,\step*rnd)}, \step*rnd - 0.8*\step)$) {};
        
        \draw[->, color0, thick] (start) -- (end.center);
        
        \node (start) at (end) {};
        \pgfmathparse{\step-0.1}
        \xdef\step{\pgfmathresult}
}
\draw[->, color0, thick] (end) -- (fbn.center);

\end{tikzpicture}
        \caption{}
        \label{fig:intuition_sketch_c}
    \end{subfigure}
    \caption{(a) Sketch of trajectory in model space $\modelspace{}$ during training. In the convergent phase, the trajectory "aims" towards the risk minimiser $\model^*$. In the divergent (overfitting) phase, the distance between the model and the risk minimiser increases. (b) When labels are noisy, the trajectory aims towards the noisy risk minimiser $\tilde{\model}^*$. Using a strictly proper loss, the clean and noisy risk minimisers differ. (c) Using a robust loss, the aim of the convergent phase is the same under the clean and noisy data distributions, but overfitting is still possible.}
    \label{fig:intuition_sketch}
\end{figure*}
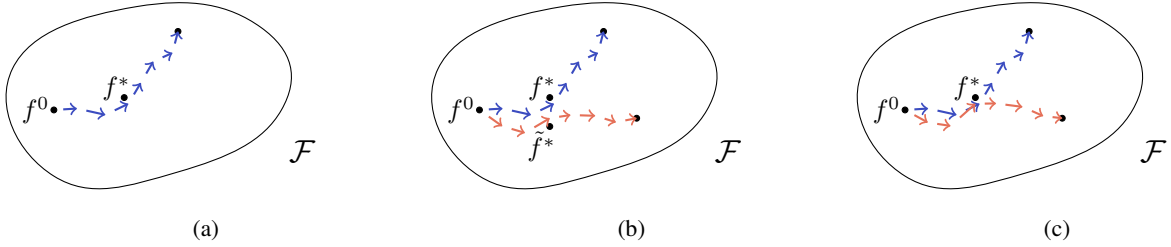

In the view of \cref{fig:intuition_sketch_b}, using a strictly proper loss function, the convergent phase of training  will "aim" towards $\model^{*}(x) = \prob(Y \mid X=x)$ in the case of noise-free data and towards $\tilde{\model}^{*}(x) = \prob(\Noisyy \mid X=x)$ if labels are noisy.  
As a first contribution of this paper we therefore:
\ifarxiv
\vspace{-1ex}
\else
\vspace{-5ex}
\fi
\paragraph{Characterise the Properties of $\mathbf{\mathbbb{P}(\Noisyy \mid X)}$ Relative to $\mathbf{\mathbbb{P}(Y \mid X)}$.} We show (\cref{prop:decision_boundaries}) that under the commonly used symmetric label noise assumption, as well as for a more realistic input-dependent noise, $\tilde\model^*(x) = \prob(\Noisyy \mid X=x)$ shares decision boundaries with $\model^*(x) = \prob(Y \mid X=x)$. This is encouraging if accuracy is the main quantity of interest, and can explain observed robustness in accuracy of neural networks trained with regularisation techniques such as early stopping \parencite{Li2020} and pre-training \parencite{Hendrycks2019}. 
Still, for reliable uncertainty quantification, this is not enough. Considering the two risk minimisers, we show (\cref{prop:conditional_entropy}) that $\tilde{\model}^*$ has higher entropy than $\model^*$. While this is perhaps not surprising, since the additional noise can only increase the entropy, it still shows that we will not be able to recover the true probability $\prob(Y  \mid X)$, despite the apparent robustness in accuracy. Of a higher interest, and with more severe practical implications, is that the noisy risk minimiser $\tilde \model^*$ is not only insufficient for recovering $\prob(Y  \mid X)$, but it also fails to be calibrated (\cref{prop:noisyy_calibration}). We demonstrate in \cref{fig:intro_example} how the predicted class probability of a simple neural network trained with early stopping is affected by the presence of symmetric label noise in the training data. The models trained with and without label noise have similar accuracy on clean test data (0.984 vs. 0.992), but the uncertainty clearly increases with the addition of the noise. 

\noindent
Building on these findings, as a second contribution we also:
\ifarxiv
\vspace{-2ex}
\else
\vspace{-5ex}
\fi
\paragraph{Critically Review the Use of Robust Loss Functions.}
Employing a robust loss function \parencite{Ghosh2015, Ghosh2017}, means that $\tilde \model^* = \model^*$, i.e., as illustrated in \cref{fig:intuition_sketch_c}, the training trajectories "aim" for the same point in the initial training phase. However, we argue that this is not enough when it comes to reliability. Indeed, we show (\cref{prop:robustness_strictly_proper}) that robust loss functions are never strictly proper, so we can not expect them to accurately recover $\prob(Y  \mid X)$. Furthermore, to relax this strong requirement, we define a weaker notion of a calibration-based strictly proper loss function (\cref{def:calibration_based_sp}), and show (\cref{prop:robustness_calibrated}) that the robustness condition is insufficient for a loss function to be calibration-based strictly proper. Specifically, the set of symmetric, robust loss functions \parencite{Ghosh2015, Ghosh2017} are never calibration-based strictly proper. This is of high relevance since, to our knowledge, this is the only identified class of loss functions that are robust to simple non-uniform label noise and that also does not require estimation of the noise distribution. 

Finally, we demonstrate empirically (\cref{sec:mae_test}) that models trained with robust loss functions are \emph{not} robust to overfitting and that any observed robustness does not follow from the theory. Indeed, robustness, in this regard, is a property related to the asymptotic risk minimiser $\model^*$. In practice, 
a model can overfit to the noise in the training data even when a robust loss function is used. 

In summary, loss robustness concerns the risk minimiser $\model^*$ which determines the "aim" of the convergent phase of the learning trajectory. However, our results show that there is limited theoretical support for why this target point should be any better using a robust (\cref{fig:intuition_sketch_c}) compared to a strictly proper loss function (\cref{fig:intuition_sketch_b}). Specifically, in the case of a symmetric, robust loss functions, it holds that both are robust in terms of accuracy, none of them are robust when it comes to uncertainty quantification (whether we use the stronger notion of recovering $\prob(Y  \mid X)$ or the weaker notion of calibration), 
and the practical issue of potentially overfitting to label noise remains in both cases.

Furthermore, while our results can be used to explain perceived robustness of specific training algorithms, they also point towards a weakness in evaluating robustness solely in terms of accuracy. They demonstrate that robustness in accuracy does not imply reliability, or, more specifically, robustness in uncertainty quantification.
At the same time, uncertainty-agnostic metrices such as accuracy are commonly used to evaluate robustness against label noise \parencite{Song2020} while, to our understanding, uncertainty quantification is consistently overlooked. 
Our conclusion is that further investigation is needed to better understand the effect of label noise on model reliability, as well as the training dynamics. Moreover, we suggest that uncertainty quantification should have a natural part in the evaluation of noise-robust training algorithms. For future work, one potential direction is that of developing loss functions that ensure robustness in accuracy as well as calibration-based strictly properness.

\begin{figure*}[t]
\centering
    \tikzset{
  font={\fontsize{16pt}{12}\selectfont}}

\definecolor{color0}{rgb}{0.705673158,0.01555616,0.150232812}
\definecolor{color1}{rgb}{0.2298057,0.298717966,0.753683153}

\ifarxiv
\newcommand\vertsep{4}
\else
\newcommand\vertsep{4}
\fi

\begin{tikzpicture}[scale=0.45] 

\begin{groupplot}[group style={group size=3 by 1, group name=toy example plots, horizontal sep =\vertsep cm,vertical sep =2 cm }]

\nextgroupplot[
legend cell align={left},
legend style={fill opacity=0.8, draw opacity=1, text opacity=1, draw=white!80!black},
tick align=outside,
tick pos=left,
x grid style={white!69.0196078431373!black},
xlabel={Dim 1},
xlabel style={yshift=-0.2cm},
xmin=-1.525, xmax=1.475,
xtick style={color=black},
y grid style={white!69.0196078431373!black},
ylabel={Dim 2},
ylabel style={yshift=-0.2cm},
legend image post style={scale=8},
ymin=-1.525, ymax=1.475,
ytick style={color=black}
]

\input{background/figures/raw_figures/toy_data_2}

\nextgroupplot[
tick align=outside,
tick pos=left,
title={$f_1(x)$},
x grid style={white!69.0196078431373!black},
xlabel={Dim 1},
xlabel style={yshift=-0.2cm},
xmin=-1.505, xmax=1.495,
xtick style={color=black},
y grid style={white!69.0196078431373!black},
ylabel={Dim 2},
ylabel style={yshift=-0.2cm},
ymin=-1.505, ymax=1.495,
ytick style={color=black}
]
\addplot graphics [includegraphics cmd=\pgfimage,xmin=-1.505, xmax=1.495, ymin=-1.505, ymax=1.495] {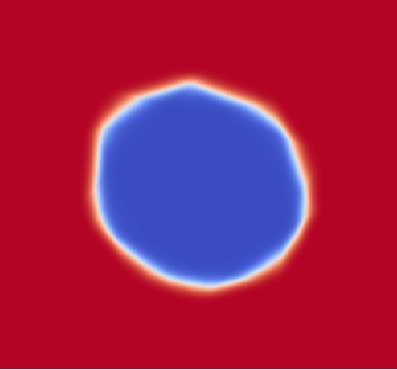};

\nextgroupplot[
colorbar,
colorbar style={ylabel={}},
colormap={mymap}{[1pt]
  rgb(0pt)=(0.2298057,0.298717966,0.753683153);
  rgb(1pt)=(0.26623388,0.353094838,0.801466763);
  rgb(2pt)=(0.30386891,0.406535296,0.84495867);
  rgb(3pt)=(0.342804478,0.458757618,0.883725899);
  rgb(4pt)=(0.38301334,0.50941904,0.917387822);
  rgb(5pt)=(0.424369608,0.558148092,0.945619588);
  rgb(6pt)=(0.46666708,0.604562568,0.968154911);
  rgb(7pt)=(0.509635204,0.648280772,0.98478814);
  rgb(8pt)=(0.552953156,0.688929332,0.995375608);
  rgb(9pt)=(0.596262162,0.726149107,0.999836203);
  rgb(10pt)=(0.639176211,0.759599947,0.998151185);
  rgb(11pt)=(0.681291281,0.788964712,0.990363227);
  rgb(12pt)=(0.722193294,0.813952739,0.976574709);
  rgb(13pt)=(0.761464949,0.834302879,0.956945269);
  rgb(14pt)=(0.798691636,0.849786142,0.931688648);
  rgb(15pt)=(0.833466556,0.860207984,0.901068838);
  rgb(16pt)=(0.865395197,0.86541021,0.865395561);
  rgb(17pt)=(0.897787179,0.848937047,0.820880546);
  rgb(18pt)=(0.924127593,0.827384882,0.774508472);
  rgb(19pt)=(0.944468518,0.800927443,0.726736146);
  rgb(20pt)=(0.958852946,0.769767752,0.678007945);
  rgb(21pt)=(0.96732803,0.734132809,0.628751763);
  rgb(22pt)=(0.969954137,0.694266682,0.579375448);
  rgb(23pt)=(0.966811177,0.650421156,0.530263762);
  rgb(24pt)=(0.958003065,0.602842431,0.481775914);
  rgb(25pt)=(0.943660866,0.551750968,0.434243684);
  rgb(26pt)=(0.923944917,0.49730856,0.387970225);
  rgb(27pt)=(0.89904617,0.439559467,0.343229596);
  rgb(28pt)=(0.869186849,0.378313092,0.300267182);
  rgb(29pt)=(0.834620542,0.312874446,0.259301199);
  rgb(30pt)=(0.795631745,0.24128379,0.220525627);
  rgb(31pt)=(0.752534934,0.157246067,0.184115123);
  rgb(32pt)=(0.705673158,0.01555616,0.150232812)
},
point meta max=1,
point meta min=0,
tick align=outside,
tick pos=left,
title={$f_1(x)$},
x grid style={white!69.0196078431373!black},
xlabel={Dim 1},
xlabel style={yshift=-0.2cm},
xmin=-1.505, xmax=1.495,
xtick style={color=black},
y grid style={white!69.0196078431373!black},
ylabel={Dim 2},
ylabel style={yshift=-0.2cm},
ymin=-1.505, ymax=1.495,
ytick style={color=black}
]
\addplot graphics [includegraphics cmd=\pgfimage,xmin=-1.505, xmax=1.495, ymin=-1.505, ymax=1.495] {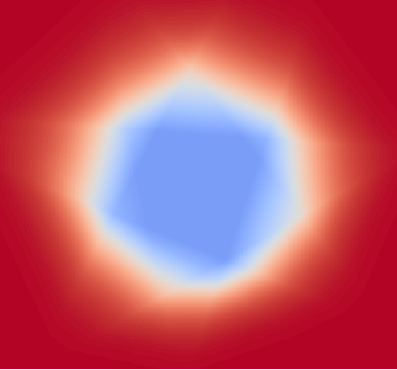};

\end{groupplot}
\end{tikzpicture}
    \caption{Predicted class 1 probabilities, $f_1(x)$, over the input space, of models trained on circle data. Left: clean training data \parencite{Pedregosa2011}. Middle: model trained on clean data. Right: model trained on noisy data with label flip probability 0.2.}
    \label{fig:intro_example}    
\end{figure*}

\section{Preliminaries}
\label{sec:preliminaries}
We will consider classification problems with $K$ classes, input variable $X \in \xspace$ 
and label $Y \in \yspace = \{1, \dots, K\}$. 
We will denote the true vector of probabilities over outcomes in $\yspace$ given the observation $X=x$ by $\truemodel(x)$ with elements $\truemodel_k (x) = \prob(Y=k \mid X=x)$. When there is label noise in the data, we observe $\Noisyy \in \yspace$ in place of $Y$ and $\noisymodel(x)$ is used to denote the probability vector with entries $\noisymodel_k (x) = \prob(\Noisyy=k \mid X=x)$. 

For the generative process of $\Noisyy$, we will consider a version of input-dependent noise referred to as \textit{simple non-uniform} label noise \parencite{Ghosh2017}. 

\begin{definition} [Simple non-uniform label noise, see e.g. \cite{Ghosh2017}] For simple non-uniform label noise,
\begin{align*}
    \prob(\Noisyy=\noisyy \mid Y=y, X=x) = 
    \begin{cases}
        1 - \omega(x), \quad\;\; \text{if} \; \noisyy=y \\
        \frac{\omega(x)}{K-1}, \qquad\;\;\;\;\;\! \text{otherwise}
    \end{cases}
\end{align*}
where the flip probabilities are defined by an input-dependent noise parameter $0 \leq \omega(x) < \frac{K-1}{K}$.
\label{def:simple_non_uniform_noise}
\end{definition}
In the literature, it is common to assume that label noise is input-independent, see \parencite{Song2020} and the references therein. Hence, although referred to as "simple" the simple non-uniform label noise assumption is still more complex than what is often assumed, since it allows the noise to vary across the input space. In parallel, this noise assumption does not exclude the possibility of input-independence. For instance, the frequently used \textit{symmetric} label noise \parencite{Song2020}, for which $\omega(x) = \omega \sforall x$, is a special case of simple non-uniform label noise.

The bounds on $\omega(x)$ given in \cref{def:simple_non_uniform_noise} will be implicitly assumed throughout the paper. The upper bound on the noise parameter $\omega(x)$ ensures that the noisy label has a higher probability of being equal to the true label than it has of being equal to any other label. Theoretically, this will preserve the dominant label in each cluster of a sampled data set. Throughout the paper, we will assume that $\noisymodel(X) \neq \truemodel(X)$ with probability larger than 0. For simple non-uniform label noise, this is equivalent to assuming that $\prob\left(\{\omega(X) > 0\} \cap \{\truemodel(X) \neq \frac{1}{K} \cdot \mathbf{1}_K\} \right) > 0$, where $\mathbf{1}_K$ is the vector of ones of size $K$. For the complement, the label noise is effectively non-existent. 

\subsection{Risk Minimisation and Reliability}

The aim is to train a model $\model: \xspace \rightarrow \Delta^{K-1}$, belonging to some model class $\modelspace{}$ and predicting a conditional distribution over $Y$ for each $x \in \xspace$. Here, $\Delta^{K-1}$ denotes the $(K-1)$-simplex and so, $\model_{k}(x) \in [0, 1] \sforall k$ and $\|\model (x) \|_1 = \sum_{k \in \yspace} \model_{k}(x) = 1$. 

We will consider risk minimisation, where the risk with respect to the data distribution $\prob(X, Y)$ is defined according to 
\ifarxiv
\begin{align}
    \risk (\model) = \expectation_{X,Y}[\loss (\model(X), Y)] =
    \expectation_{X}[\expectation_{Y\mid X} [\loss (\model(X), Y)]]
    \label{eq:risk}
\end{align}
\else
\begin{align}
    \risk (\model) = 
    \expectation_{X}[\expectation_{Y\mid X} [\loss (\model(X), Y)]]
    \label{eq:risk}
\end{align}
\fi
with $\loss (\model(X), Y)$ a predefined loss function. In practice, the true data distribution is unknown and the risk is approximated by the empirical risk with respect to a finite data set.

Optimally, we would minimise $\risk$ directly. However, when noisy labels are observed, minimising the noisy risk $\noisyrisk$, with respect to the noisy data distribution $\prob(X, \Noisyy)$, is arguably the most straightforward, and occasionally the most sensible, option. In this case, the aim is nevertheless to find a model that performs well on the intended, clean target distribution. Hence, we will consider the properties of the model obtained by minimising $\noisyrisk$ and its relation to $Y$. 

Central to the arguments put forth in this paper, is the idea that reliability is a model property that can be equally important as accuracy. Formally, reliability (or calibration) is defined as follows.

\begin{definition}[Calibrated model, see e.g. \cite{Brocker2009, Vaicenavicius2019}]
     Let $\model: \xspace \rightarrow \Delta^{K-1}$ be a probabilistic predictive model and assume that $\prob(Y \mid \model (X))$ exists. The model $\model$ is calibrated if
    \begin{equation*}
            \prob(Y \mid \model(X)) = \model(X) 
    \end{equation*}
    almost surely.
    \label{def:calibrated_def}
\end{definition}
Hence, a model is reliable if its confidence, represented by the  predicted class probabilities, is equal to the true conditional probability over the outcome. 
\subsection{Proper and Robust Loss Functions}

In risk minimisation, the use of a strictly proper loss function gives asymptotic, theoretical guarantees that the true conditional distribution over labels will be recovered at the minimum. 

\begin{definition}[Proper loss function, see e.g. \cite{Gneiting2007, Reid2010}] 
    A loss function $\loss$ is proper if
    \begin{align*}
        \truemodel(x) \in \underset{\model(x) \in \Delta^{K-1}}{\text{\normalfont{argmin}}} \; \mathbb{E}_{Y\mid X=x}[\loss(\model (x), Y)]
    \end{align*}
    for all conditional distributions $\prob(Y\mid X)$. If the minimum at $\truemodel(x)$ is unique, $\loss$ is strictly proper.  
    \label{def:proper_loss}

\end{definition}

\Cref{def:proper_loss} concerns the point-wise risk of proper loss functions. However, by minimising the point-wise risk, we implicitly minimise the full risk. We will denote the set of proper and strictly proper loss functions by $\losspace{P}$ and $\losspace{SP}$, respectively. 

Under the presence of label noise, the asymptotic risk minimiser of a strictly proper loss function will be $\noisymodel$ as opposed to $\truemodel$. Hence, the risk minimiser will differ depending on if $Y$ or $\Noisyy$ is considered. In contrast to this, robust loss functions have been identified and developed based on the idea of achieving robustness through noise-insensitive risk minimisers (e.g. \parencite{Ghosh2015, Ghosh2017, Wang2019, Zhang2018}). In this context, a loss $\loss$ is said to be robust to label noise if the asymptotic risk minimisers of the clean and noisy risks have the same probability of misclassification. A sufficient condition is that the risk minimiser, $\model^*$, of the clean risk, is also a minimiser of the noisy risk \parencite{Ghosh2017}. 
\begin{definition}[Robust loss function, see \cite{Ghosh2015, Ghosh2017}] A loss function $\loss$ is robust to label noise if for all asymptotic minimisers $\model^*$ of the clean risk, $\risk$, it holds that
    \begin{align*}
        \noisyrisk (\model^*) \leq \noisyrisk (\model), \quad \sforall \model \in \modelspace{}
    \end{align*}
    where $\noisyrisk$ is the risk under the noisy data distribution.
    
    \label{def:robustness_condition}
\end{definition}

We will refer to the set of robust loss functions by $\losspace{R}$. 

\Cref{def:robustness_condition} gives a theoretical condition for robustness, but does not, in itself, tell us how to construct a robust loss function.
Hence, there is a practical need of identifying individual loss functions, or classes thereof, that fulfills this condition. \cite{Ghosh2015, Ghosh2017}, find that \textit{symmetric} loss functions are robust under symmetric label noise and under simple non-uniform label noise with the extra condition that $\loss$ is positive and $\risk (\model^*) = 0$. 
\begin{definition}[Symmetric loss function, see e.g. \cite{Ghosh2015, Ghosh2017}]  A loss function $\loss$ is symmetric if
    \begin{equation*}
        \sum_{k=1}^{K} \loss (q, k) = C, \quad \sforall q \in \Delta^{K-1}  
    \end{equation*}
    for some constant $C$.
    \label{def:symmetric_loss}
\end{definition}
Referring to the set of symmetric loss functions by $\losspace{S}$, it thus holds that $\losspace{S} \subseteq \losspace{R}$. To our knowledge, $\losspace{S}$ is the only identified set of loss functions that are robust to simple non-uniform label noise and that does not require an estimate of $\omega(x)$. We recognise that there exists other robust loss functions but that is either only robust to input-independent label noise \parencite{Xu2019} or rely on knowledge of $\omega(x)$ (see e.g. \cite{Patrini2017, Natarajan2018}). We refer to the supplementary material for an analysis of the information-theoretic loss function proposed in \parencite{Xu2019}. For the second group of loss functions, we do not consider them in this paper, with the argument that the true noise rates are seldomly known in practice.

\section{Main results}
We evaluate the use of strictly proper and robust loss functions in the presence of label noise by analysing their respective risk minimisers. In general, we consider the following asymptotic properties of a loss function under the influence of label noise:
\begin{enumerate}[label=(\Alph*)]
    \item \label{itm:first_cond} It preserves decision boundaries, or accuracy, of a model trained with clean data.
    \item \label{itm:second_cond} It recovers the true conditional probability $\truemodel(x)$.
    \item \label{itm:third_cond} It results in a reliable, or calibrated, model.
\end{enumerate}
We find that both strictly proper and robust loss functions fulfill \cref{itm:first_cond}, but neither gives asymptotic, theoretical guarantees for \cref{itm:second_cond} or \cref{itm:third_cond}. Moreover, strictly proper as well as robust loss functions are susceptible to overfitting in practice. In parallel, while \cref{itm:first_cond} is commonly considered in the context of label noise robustness, \cref{itm:second_cond} and \cref{itm:third_cond} are consistently overlooked. This, we argue, in spite of their equal importance in many practical applications. The complete proofs and derivations for the results presented in this section can be found in the supplementary material.

\subsection{Distribution Over Noisy Labels}
Using a strictly proper loss function, we expect a flexible model to asymptotically approximate $\truemodel$ under the clean data distribution and $\noisymodel$ if label noise is present in the data. Hence, we expand on the implications of training a model with a data set containing simple non-uniform label noise by deriving results regarding the conditional distribution over noisy labels.
To evaluate  $\noisymodel$, we first note that it can be derived from $\truemodel$ by marginalisation
\begin{align}
    \noisymodel(x) = \sum_{k=1}^{K} \prob(\Noisyy \mid Y=k, X=x) \truemodel_k (x).
\end{align}
An alternative formulation is introduced in the following lemma.

\begin{restatable}{lemma}{glemma}
    The conditional probability vector $\noisymodel(x)$ can be written as function of $\truemodel (x)$ according to
    \begin{equation*}
        \noisymodel(x) = 
        \left(1-\frac{\omega(x)K}{K-1}\right)\truemodel(x) + \frac{\omega(x)}{K-1} \cdot \mathbf{1}_K
    \end{equation*}
    where $\mathbf{1}_K$ is the vector of ones with length $K$. Moreover, it holds for any two classes $i, j \in \yspace$ that $\noisymodel_i (x) > \noisymodel_j (x)$ if and only if $\truemodel_i (x) > \truemodel_j (x)$.
    \label{lem:func_of_g}
\end{restatable}

From \cref{lem:func_of_g}, the first result regarding the properties of $\noisymodel$ follows.

\begin{restatable}{proposition}{equaldecisions}
    Assume that the prediction is taken as the most probable class, then $\noisymodel$ has the same decision boundaries as $\truemodel$.
    \label{prop:decision_boundaries}
\end{restatable}
\Cref{prop:decision_boundaries} establishes that the use of a strictly proper loss function will asymptotically preserve the accuracy of $\truemodel$ in the presence of simple non-uniform label noise. It is noteworthy, that if $\truemodel$ and $\noisymodel$ had not shared decision boundaries, no classification-calibrated loss function would be robust to the noise in terms of accuracy. Simply put, a loss function is classification-calibrated if the class predictions of its asymptotic risk minimiser(s) corresponds to taking the most probable class with respect to the true conditional probability of the observed target variable \parencite{Bartlett2006}. In binary classification, it is a minimal condition commonly imposed on surrogate losses of the 0/1-risk.

Next, we will show that preserving accuracy is not enough if uncertainty quantification is critical. As a first step towards this realisation, we consider the entropy of $\noisymodel$. The conditional entropy of a probabilistic vector $\model(x)$ given $X=x$ is defined as
\begin{align}
 \entropy{\model(x)} = - \sum_{k=1}^{K} \model_k(x) \log \model_k(x).
 \label{eq:entropy_def}
\end{align}
We have the following result. 

\begin{restatable}{proposition}{higherentropy}
    The average conditional entropy of $\noisymodel$ is higher than that of $\truemodel$,  that is $\expectation_X[\entropy{\noisymodel(X)}] > \expectation_X[\entropy{\truemodel(X)}]$.
    \label{prop:conditional_entropy}
\end{restatable}

Since $\noisymodel(X) \neq \truemodel(X)$ with non-zero probability, \cref{prop:conditional_entropy} implies that training with a strictly proper loss function, under the influence of label noise, will not recover the true conditional probability over clean labels. This does not necessarily entail that $\noisymodel$ is uncalibrated. Nevertheless, the next result states that it is.

\begin{restatable}{proposition}{noisycalibration}
    The vector of conditional probabilities $\noisymodel (X)$ is not calibrated with respect to the distribution
    $\prob(Y \mid X)$ over clean labels.
    \label{prop:noisyy_calibration}
\end{restatable}
We conclude that \cref{prop:decision_boundaries} is in agreement with the observed robustness to label noise in accuracy when training with regularisation \parencite{Li2020, Hendrycks2019}. If the model does not overfit to the data, such that it approximates $\noisymodel$, it will have similar accuracy on noise-free data as a model approximating $\truemodel$. However, \cref{prop:conditional_entropy,prop:noisyy_calibration} imply that this does not generalise to uncertainty quantification. Training with noise in the data using a strictly proper loss, will not asymptotically recover $\truemodel$ nor result in a reliable model. Hence, evaluating robustness solely in terms of accuracy can give a perceived robustness against label noise in spite of the final model not being reliable.

\subsection{Evaluation of Robust Loss Functions}
We derive results concerning the set of robust loss functions in \cref{def:robustness_condition}. Robustness in this regard has previously been determined to be sufficient for preserving accuracy \parencite{Ghosh2015, Ghosh2017}. Hence, our analysis concentrates on the reliability of the asymptotic risk minimisers of robust loss functions, shared between the clean and noisy risks. 
We first establish that robust loss functions are not strictly proper. 
\begin{restatable}{proposition}{robustnotsp}
    Robust loss functions (\cref{def:robustness_condition}) are not strictly proper.
    \label{prop:robustness_strictly_proper}
\end{restatable}
Intuitively, if a robust loss function had been strictly proper, it should have had a unique minimum at $\truemodel$ under the clean data distribution and at $\noisymodel$ under the noisy data distribution and hence, there would be no overlap in risk minimisers unless $\noisymodel = \truemodel$.  The opposite of \cref{prop:robustness_strictly_proper} must also be true; strictly proper loss functions are not robust according to \cref{def:robustness_condition} (a proof for $K=2$ can be found in \parencite{Reid2010}). Thus, $\losspace{SP} \cap \losspace{R} = \emptyset$.

Next, we consider the reliability of the asymptotic risk minimisers of robust loss functions. First, we again point out that, to our knowledge, the only identified class of loss functions that fulfills \cref{def:robustness_condition} under simple non-uniform label noise and at the same time does not require an estimate of the noise parameter $\omega(x)$ is that of symmetric loss functions \parencite{Ghosh2015, Ghosh2017}. Hence, the following results focus on this class.

We have shown that, in the context of label noise, using a strictly proper loss function is not optimal if reliability is of importance. Hence, it is of higher relevance to investigate if a robust loss function can recover $\truemodel$, than if it is strictly proper. We argue, however, that this is not an inherent property of robust loss functions in general.
The conditional (or point-wise) risk minima $\model ^* (x) = [\model^*_1 (x), 1-\model^*_1 (x)]^\top$ for a symmetric loss function, in a binary classification setting, are found at
\begin{align}
    \quad \model^*_1 (x) = \gamma \mathbb{I}_{\prob(Y=1\mid X=x) \geq \frac{1}{2}} + \gamma' \mathbb{I}_{\prob(Y=1\mid X=x) < \frac{1}{2}}
    \label{eq:symmetric_loss_minima}
\end{align}
with $\gamma \in \text{argmin}_{q \in [0,1]}  \loss(q, 1)$ and $\gamma' \in \text{argmax}_{q \in [0,1]}  \loss(q, 1)$. From here on, we will assume that $\gamma \in [\frac{1}{2}, 1]$ and $\gamma' \in [0, \frac{1}{2})$, since in all other cases for which $\model^*_1 (x) \in [0, 1]$, predictions will be consistently incorrect for at least one class.
In agreement with \parencite{Charoenphakdee2019}, we can not mathematically recover the true conditional probability over $Y$ from the minima in \cref{eq:symmetric_loss_minima}. As a result, $\truemodel$ can not be a unique minimiser of the clean (or noisy) risk of a symmetric loss function in general. Consequently, \cref{def:robustness_condition} is not sufficient for recovering $\truemodel$.  

\label{sec:robust_loss_calibration_SP}
While an arbitrary robust loss function will not asymptotically recover $\truemodel$, this does not exclude the possibility of obtaining a calibrated model. Unfortunately, we demonstrate that the two properties of fulfilling the robustness condition in \cref{def:robustness_condition} and having only calibrated risk minimisers do not coincide. To this end, we will introduce a new set of loss functions referred to as \textit{calibration-based strictly proper}. We denote this set of loss functions by $\losspace{CSP}$. 
\begin{definition}(Calibration-based strictly proper loss function)
    Let $\modelspace{C}$ be the set of calibrated models in $\modelspace{}$. 
    The loss function $\loss$, with asymptotic risk minimisers $\model^* \in \modelspace{}$, is calibration-based strictly proper if 
    \begin{align*}
        \model^* \in \modelspace{C}, \quad \sforall \model^* 
    \in \modelspace{},
    \end{align*} 
    for all $\prob(Y \mid X)$ and for all input distributions $\mu_X$.
    \label{def:calibration_based_sp}
\end{definition}
In parallel to the definition of a strictly proper loss function, the definition of a calibration-based strictly proper loss puts a restriction on the asymptotic minimiser(s) of the corresponding risk. However, instead of requiring that the (unique) risk minimiser is equal to $\truemodel$, it requires that all asymptotic risk minimisers are calibrated. Hence, employing $\loss \in \losspace{CSP}$ will, asymptotically, result in a calibrated model. Since $\truemodel$ is calibrated by definition, all strictly proper loss functions are calibration-based strictly proper. However, this is not true for symmetric loss functions.
\begin{restatable}{proposition}{symmetriccalibration}
    Symmetric loss functions (\cref{def:symmetric_loss}) are not calibration-based strictly proper. 
    \label{prop:robustness_calibrated}
\end{restatable}
From \cref{prop:robustness_calibrated}, we have $\losspace{S} \cap \losspace{CSP} = \emptyset$. As an example, mean absolute error and sigmoid loss, both symmetric, have a unique minimum for $\model_1(x) \in [0, 1]$ in \cref{eq:symmetric_loss_minima} with $(\gamma, \gamma') = (1, 0) \sforall x \in \xspace$. For a non-separable classification problem, this model is never calibrated. Nevertheless, symmetric loss functions are robust to symmetric label noise also in this case. 

To conclude, \cref{def:robustness_condition} constitutes a sufficient condition for robustness in accuracy. However, there are no theoretical basis for why, even asymptotically, the use of a robust loss function will result in a reliable model. On the contrary, the class of symmetric, robust loss functions will \textit{not} (asymptotically) recover $\truemodel$ and are \textit{not} calibration-based strictly proper.

\subsection{Robustness and Overfitting}
\label{sec:mae_test}
While \cref{def:robustness_condition} relies on asymptotic theory, a finite training set has to suffice in practice. Naturally, empirical evaluation has been used to demonstrate the noise-insensitivity and motivate the use of robust loss functions also under such circumstances \parencite{Ghosh2017}. However, we show, with a similar experiment, that these loss functions are not robust to overfitting and argue that any perceived robustness is not explained by the asymptotic theory.\footnote{
Code provided at:
\url{https://github.com/AOlmin/robustness_and_reliability_in_weak_supervision}}

For the experiment, we consider neural networks with one hidden layer of 500 hidden units and with LeakyReLU as activation function. The data used is flattened MNIST images \parencite{Lecun1998} to which we artificially add symmetric label noise with parameter $\omega \in [0.0, 0.3, 0.5]$. We train the models using ADAM optimization \parencite{Kingma2015} with a constant learning rate of 0.005 and a batch size of 100. Similar to \cite{Ghosh2017}, the robust (and symmetric) loss function that we consider is mean absolute error (MAE). First, we train models with MAE and random initialisation. We evaluate the accuracy on the corresponding training data set and a noise-free test data set during the course of training, as shown in \cref{fig:acc_MAE}. For comparison, we do a similar evaluation of models trained with categorical cross-entropy (CCE) loss, see \cref{fig:acc_CCE}.   

Under the influence of label noise, the models trained with CCE clearly overfit to the training data. The accuracy on the training data set is close to 1 for all considered values of $\omega$ within 500 epochs of training, while the test accuracy decreases as training progresses. In contrast, the models trained with MAE seems to stabilise at a point where the accuracy over the clean test data set remains high, even in the presence of label noise. Observing only these trends, it is intriguing to assume that MAE is robust to overfitting and that this can be explained by the loss function's noise-insensitive risk minimiser. However, overfitting is not considered in asymptotic theory, on which the noise-insensitivity of robust loss functions rely. In addition, we have previously demonstrated that the asymptotic risk minimiser of CCE loss, a strictly proper loss, is also robust in accuracy to symmetric label noise. In spite of this, the models trained with CCE loss overfit to the label noise.

To support our arguments, we train models again with MAE but replace the random initialisation. For each noise level, the model weights are initialised with those obtained when training a model with CCE loss for 500 epochs. Evaluating the new models in terms of accuracy, the trends observed are more similar to those of the models trained only with CCE loss, see \cref{fig:acc_MAE_pretrained}. Evidently, the models are capable of overfitting to training data, even when a robust loss function is used. Furthermore, in all cases, the models achieve a smaller training loss compared to those trained with random initialisation, as shown in  \cref{fig:mae_test_loss}. Hence, the models trained with random initialisation must be stuck in local minima. At the same time, the asymptotic theory of robust loss functions concerns global, not local, minima.

To conclude, we have demonstrated that robustness in the context of \cref{def:robustness_condition} should not be confused with robustness to overfitting. Indeed, the asymptotic theory behind robust loss function is not concerned with this issue. Thus, models can overfit to label noise, even when a robust loss function is employed. 

\begin{figure*}
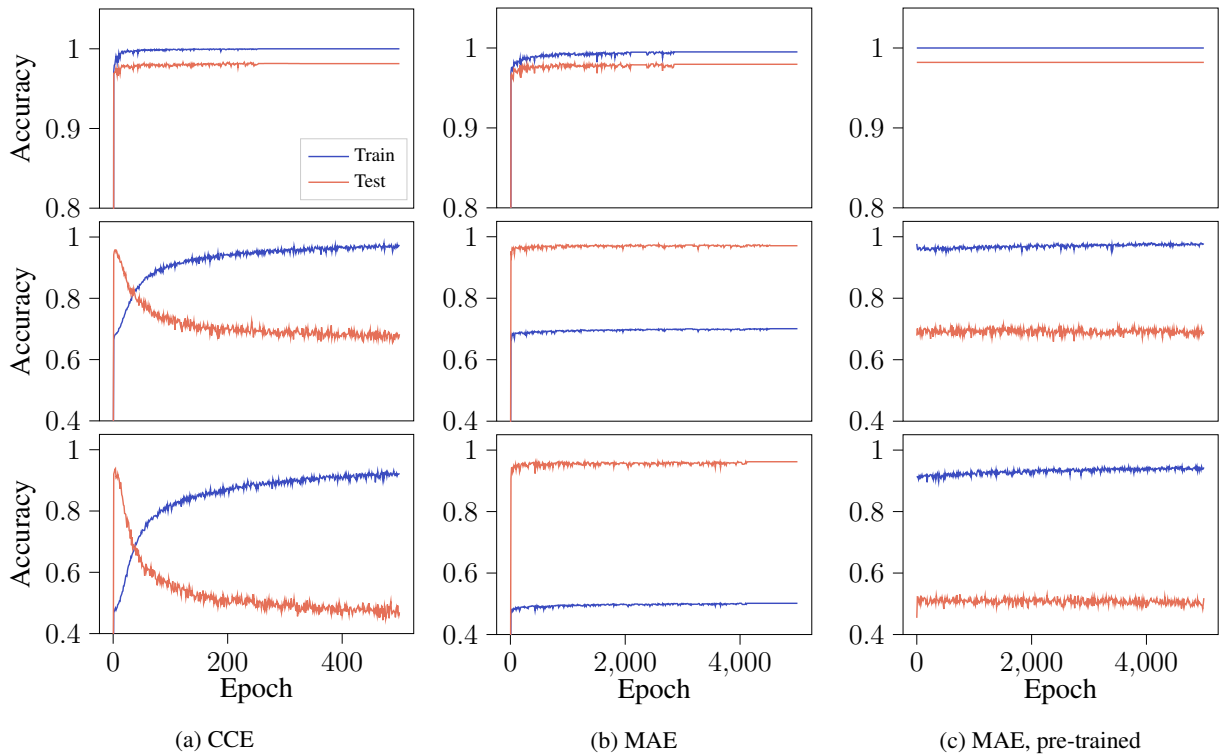

    \centering
    \begin{subfigure}{0.33\textwidth}
\tikzset{font={\fontsize{12pt}{12}\selectfont}}

\ifarxiv
\newcommand\figscale{0.9}
\newcommand\vertsep{0.2}
\newcommand\legendfont{8}
\else
\newcommand\figscale{0.8}
\newcommand\vertsep{0.2}
\newcommand\legendfont{10}
\fi

\begin{tikzpicture}[scale = \figscale]

\definecolor{color0}{rgb}{0.2298057,0.298717966,0.753683153}
\definecolor{color1}{rgb}{0.89904617,0.439559467,0.343229596}
\definecolor{color2}{rgb}{0.2298057,0.298717966,0.753683153}
\definecolor{color3}{rgb}{0.89904617,0.439559467,0.343229596}
\definecolor{color4}{rgb}{0.2298057,0.298717966,0.753683153}
\definecolor{color5}{rgb}{0.89904617,0.439559467,0.343229596}

\pgfplotsset{%
            width=1.1\linewidth,
            height=0.8\linewidth
            }

\begin{groupplot}[group style={group size=1 by 3, group name=omega 0 acc plots, horizontal sep =1.5cm,vertical sep =\vertsep cm}]
\nextgroupplot[
legend cell align={left},
legend style={
  fill opacity=0.8,
  draw opacity=1,
  text opacity=1,
  at={(0.98,0.04)},
  anchor=south east,
  draw=white!80!black,
  font={\fontsize{\legendfont pt}{12}\selectfont}
},
tick align=outside,
tick pos=left,
xmin=-23.95, xmax=524.95,
xmajorticks=false,
y grid style={white!69.0196078431373!black},
ylabel={Accuracy},
ylabel style = {yshift=-0.1cm},
ymin=0.8, ymax=1.05,
ytick style={color=black}
]

\input{experiments/figures/cce_0}

\nextgroupplot[
legend cell align={left},
legend style={
  fill opacity=0.8,
  draw opacity=1,
  text opacity=1,
  at={(0.97,0.03)},
  anchor=south east,
  draw=white!80!black
},
tick align=outside,
tick pos=left,
x grid style={white!69.0196078431373!black},
xmajorticks=false,
xmin=-23.95, xmax=524.95,
y grid style={white!69.0196078431373!black},
ylabel={Accuracy},
ylabel style = {yshift=-0.1cm},
ymin=0.4, ymax=1.05,
ytick style={color=black}
]

\input{experiments/figures/cce_03}

\nextgroupplot[
legend cell align={left},
legend style={
  fill opacity=0.8,
  draw opacity=1,
  text opacity=1,
  at={(1.62,0.8)},
  anchor=south east,
  draw=white!80!black,
},
tick align=outside,
tick pos=left,
x grid style={white!69.0196078431373!black},
xlabel={Epoch},
xmin=-23.95, xmax=524.95,
xtick style={color=black},
y grid style={white!69.0196078431373!black},
ylabel={Accuracy},
xlabel style = {yshift=0.0cm},
ylabel style = {yshift=-0.1cm},
ymin=0.4, ymax=1.05,
ytick style={color=black}
]

\input{experiments/figures/cce_05}


\end{groupplot}


\end{tikzpicture}
        \caption{CCE}
        \label{fig:acc_CCE}
    \end{subfigure}
    \begin{subfigure}{0.31\textwidth}
\tikzset{font={\fontsize{12pt}{12}\selectfont}}

\ifarxiv
\newcommand\figscale{0.9}
\newcommand\vertsep{0.2}
\else
\newcommand\figscale{0.8}
\newcommand\vertsep{0.2}
\fi

\begin{tikzpicture}[scale = \figscale]

\definecolor{color0}{rgb}{0.2298057,0.298717966,0.753683153}
\definecolor{color1}{rgb}{0.89904617,0.439559467,0.343229596}
\definecolor{color2}{rgb}{0.2298057,0.298717966,0.753683153}
\definecolor{color3}{rgb}{0.89904617,0.439559467,0.343229596}
\definecolor{color4}{rgb}{0.2298057,0.298717966,0.753683153}
\definecolor{color5}{rgb}{0.89904617,0.439559467,0.343229596}

\pgfplotsset{%
            width=1.1\linewidth * 0.32/0.3,
            height=0.8\linewidth * 0.32/0.3
}

\begin{groupplot}[group style={group size=1 by 3, group name=omega 03 acc plots, horizontal sep =1.5cm,vertical sep =\vertsep cm}]
\nextgroupplot[
legend cell align={left},
legend style={
  fill opacity=0.8,
  draw opacity=1,
  text opacity=1,
  at={(0.97,0.03)},
  anchor=south east,
  draw=white!80!black
},
tick align=outside,
tick pos=left,
xmin=-239.5, xmax=5249.5,
xmajorticks=false,
y grid style={white!69.0196078431373!black},
ylabel style = {yshift=-0.2cm},
ymin=0.8, ymax=1.05,
ytick style={color=black}
]

\input{experiments/figures/mae_ri_0}

\nextgroupplot[
legend cell align={left},
legend style={
  fill opacity=0.8,
  draw opacity=1,
  text opacity=1,
  at={(0.97,0.03)},
  anchor=south east,
  draw=white!80!black
},
tick align=outside,
tick pos=left,
x grid style={white!69.0196078431373!black},
xmajorticks=false,
xmin=-239.5, xmax=5249.5,
y grid style={white!69.0196078431373!black},
ymin=0.4, ymax=1.05,
ytick style={color=black}
]

\input{experiments/figures/mae_ri_03}

\nextgroupplot[
legend cell align={left},
legend style={
  fill opacity=0.8,
  draw opacity=1,
  text opacity=1,
  at={(1.6,0.7)},
  anchor=south east,
  draw=white!80!black
},
tick align=outside,
tick pos=left,
x grid style={white!69.0196078431373!black},
xlabel={Epoch},
xmin=-239.5, xmax=5249.5,
xtick style={color=black},
y grid style={white!69.0196078431373!black},
xlabel style = {yshift=0.0cm},
ymin=0.4, ymax=1.05,
ytick style={color=black}
]

\input{experiments/figures/mae_ri_05}


\end{groupplot}


\end{tikzpicture}
        \caption{MAE}
        \label{fig:acc_MAE}
    \end{subfigure}
    \begin{subfigure}{0.31\textwidth}
\tikzset{font={\fontsize{12pt}{12}\selectfont}}

\ifarxiv
\newcommand\figscale{0.9}
\newcommand\vertsep{0.2}
\else
\newcommand\figscale{0.8}
\newcommand\vertsep{0.2}
\fi

\begin{tikzpicture}[scale = \figscale]

\definecolor{color0}{rgb}{0.2298057,0.298717966,0.753683153}
\definecolor{color1}{rgb}{0.89904617,0.439559467,0.343229596}
\definecolor{color2}{rgb}{0.2298057,0.298717966,0.753683153}
\definecolor{color3}{rgb}{0.89904617,0.439559467,0.343229596}
\definecolor{color4}{rgb}{0.2298057,0.298717966,0.753683153}
\definecolor{color5}{rgb}{0.89904617,0.439559467,0.343229596}

\pgfplotsset{%
            width=1.1\linewidth * 0.32/0.3,
            height=0.8\linewidth * 0.32/0.3
            }

\begin{groupplot}[group style={group size=1 by 3, group name=omega 05 acc plots, horizontal sep =1.5cm,vertical sep =\vertsep cm}]
\nextgroupplot[
legend cell align={left},
legend style={
  fill opacity=0.8,
  draw opacity=1,
  text opacity=1,
  at={(0.97,0.03)},
  anchor=south east,
  draw=white!80!black
},
tick align=outside,
tick pos=left,
xmin=-239.5, xmax=5249.5,
xmajorticks=false,
y grid style={white!69.0196078431373!black},
ylabel style = {yshift=-0.2cm},
ymin=0.8, ymax=1.05,
ytick style={color=black}
]

\input{experiments/figures/mae_ccei_0}

\nextgroupplot[
legend cell align={left},
legend style={
  fill opacity=0.8,
  draw opacity=1,
  text opacity=1,
  at={(0.97,0.03)},
  anchor=south east,
  draw=white!80!black
},
tick align=outside,
tick pos=left,
x grid style={white!69.0196078431373!black},
xmajorticks=false,
xmin=-239.5, xmax=5249.5,
y grid style={white!69.0196078431373!black},
ymin=0.4, ymax=1.05,
ytick style={color=black}
]

\input{experiments/figures/mae_ccei_03}

\nextgroupplot[
legend cell align={left},
legend style={
  fill opacity=0.8,
  draw opacity=1,
  text opacity=1,
  at={(1.6,0.7)},
  anchor=south east,
  draw=white!80!black
},
tick align=outside,
tick pos=left,
x grid style={white!69.0196078431373!black},
xlabel={Epoch},
xmin=-239.5, xmax=5249.5,
xtick style={color=black},
y grid style={white!69.0196078431373!black},
xlabel style = {yshift=0.0cm},
ymin=0.4, ymax=1.05,
ytick style={color=black}
]

\input{experiments/figures/mae_ccei_05}

\end{groupplot}


\end{tikzpicture}
        \caption{MAE, pre-trained}
        \label{fig:acc_MAE_pretrained}
    \end{subfigure}
    \caption{Accuracy on (noisy) train and clean test MNIST data for models trained with symmetric label noise. Models are trained with categorical cross-entropy loss (CCE) or mean absolute error (MAE). Models trained with MAE are either initialised randomly (MAE) or pre-trained with CCE loss (MAE, pre-trained). Top: $\omega = 0.0$. Middle: $\omega=0.3$. Bottom: $\omega=0.5$.}
    \label{fig:mae_test_acc}
\end{figure*}

\begin{figure*}[t]
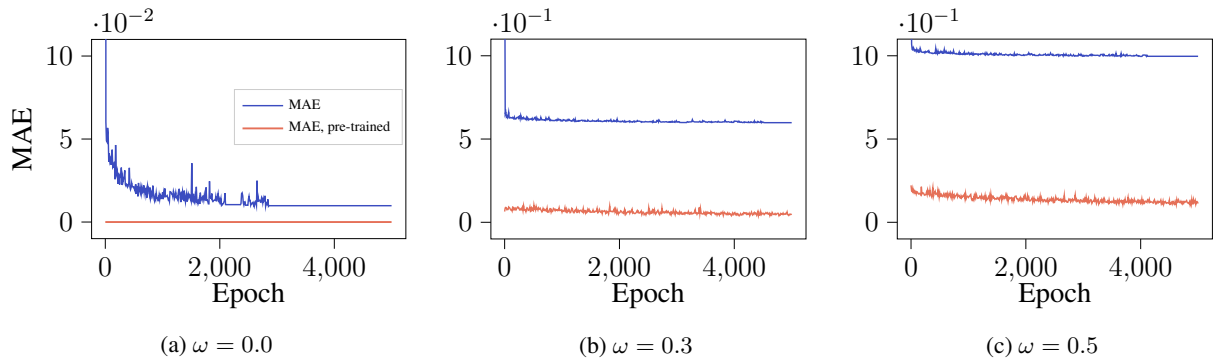

    \centering
    \begin{subfigure}{0.33\textwidth}
        \input{experiments/figures/mae_loss_0}
        \caption{$\omega = 0.0$}
    \end{subfigure}
       \begin{subfigure}{0.31\textwidth}
       \input{experiments/figures/mae_loss_03}
        \caption{$\omega = 0.3$}
    \end{subfigure}
        \begin{subfigure}{0.31\textwidth}
        \input{experiments/figures/mae_loss_05}
        \caption{$\omega = 0.5$}
   \end{subfigure}
    \caption{Train loss of models trained with MAE. The models pre-trained with CCE achieve a smaller loss than the models trained with random weight initialisation.}
    \label{fig:mae_test_loss}
\end{figure*}

\section{Discussion}
On one hand, the results presented in this paper can help explain why some training algorithms have a perceived, inherent robustness to label noise and are reassuring if accuracy is the main property of interest. On the other hand, they point towards a weakness of evaluating robustness solely based on metrices that are agnostic to predicted class probabilities. Specifically, it does not ensure that the final model is reliable. Although we have considered a simpler input-dependent noise model in this paper, the hope is that the knowledge gained could be used as a basis for gaining understanding of more complex noise. Moreover, part of the results presented are of the negative kind. Hence, if they do hold for the specific noise model considered, they also hold in the more general case.

Consider training a flexible discriminative model, such as a deep neural network, capable of approximating the true conditional probability over labels. When labels are noisy and using a strictly proper loss function, the convergent phase of the training dynamics will "aim" towards $\Tilde{\model}^{*} = \noisymodel$ instead of $\model^{*} = \truemodel$, as illustrated in \cref{fig:intuition_sketch_b}. We have shown that, in this case, $\Tilde{\model}^{*}$ share decision boundaries with $\model^{*}$, but it is not calibrated. From the view of \cref{fig:intuition_sketch_c}, using a robust loss function, the convergent phase of the training trajectory will instead aim towards the same point, regardless of whether labels are noisy or not, i.e. $\Tilde{\model}^{*}=\model^{*}$. We have characterised the properties of this risk minimiser and conclude that fulfillment of the robustness condition (\cref{def:robustness_condition}) is neither a proxy for having $\Tilde{\model}^{*}=\truemodel$ nor for $\Tilde{\model}^{*}$ being calibrated.

With these results in mind, we argue that there is \textit{no theoretical motivation} for why an arbitrary robust loss function would be better to employ than a strictly proper loss function. Under simple non-uniform noise, both are robust in terms of accuracy, but neither asymptotically guarantees that $\truemodel$ will be recovered or that the final model will be calibrated. 

In \cref{sec:robust_loss_calibration_SP}, we introduced the class of calibration-based strictly proper loss functions, for which all asymptotic risk minimisers are calibrated. Our argument is that there is a relevance in defining a loss function that, asymptotically, achieves robustness in accuracy as well as results in a calibrated model. For future work, it would be of interest to find a class of loss functions that fulfills both of these conditions and that is also not dependent on an estimate of the (usually unknown) noise distribution. 
\section{Conclusion}
\label{sec:conclusion}
For supervised training of discriminative models, we investigated the effect of label noise on model performance by analysing the properties of the conditional distribution over noisy labels. Furthermore, we critically reviewed the set of robust loss functions characterised by asymptotic risk minimisers that are insensitive to label noise. 
Under an input-dependent noise model, we found that both strictly proper and robust loss functions offer asymptotic robustness in accuracy but neither offer asymptotic, theoretical guarantees for obtaining a calibrated, or reliable, model in the presence of label noise. In addition, robustness in this context should not be misunderstood as the model being robust to noise in practice. Even when a robust loss function is used, the model can still overfit to training data. We conclude that further investigation is needed for better understanding of the effects of label noise on model performance and in order to ensure reliability of models trained with label noise. Such research would be valuable in the strive for safe employment in society.

\section*{Acknowledgements}
We thank Jakob Lindqvist and Lennart Svensson (Chalmers University of Technology, Sweden) for constructive feedback on the paper.

This research is financially supported by the Swedish Research Council via the project
\emph{Handling Uncertainty in Machine Learning Systems} (contract number: 2020-04122),
the Swedish Foundation for Strategic Research via the project
\emph{Probabilistic Modeling and Inference for Machine Learning} (contract number: ICA16-0015),
the Wallenberg AI, Autonomous Systems and Software Program (WASP) funded by the Knut and Alice Wallenberg Foundation,
and
ELLIIT.

\printbibliography

\end{refsection}

\newpage

\onecolumn
\noindent\LARGE{\textbf{Supplementary material for \textit{Robustness and \\reliability when training with noisy labels}}}
\normalsize
\appendix

\begin{refsection}

\section{Asymptotic risk minimisers}
\label{sec:pop_minimisers}

We derive the conditional (or point-wise), asymptotic risk minimisers for categorical cross-entropy (CCE) loss, mean absolute error (MAE) and sigmoid loss in a binary classification setting. The first loss function is strictly proper and the last two are symmetric according to \cref{def:symmetric_loss}. We also derive the asymptotic risk minimiser for a general symmetric loss. A summary is shown in \cref{tab:pop_minimisers}. Note that for a binary classifier, we assume $\model(x) = [\model_1(x), 1-\model_1(x)]^\top$.

To derive the risk minimisers, we consider finding the minima of the point-wise risk  $\mathcal{J}(\model(x)) = 
\expectation_{Y\mid X= x}[\loss(\model(x), Y)]$. For convenience, we will sometimes refer to the one-hot version of $Y$ by $\onehotY$. In that case, $\onehotY_k$ refers to the $k^{th}$ element of the vector. 

\begin{table}[b]
    \centering
    {\renewcommand{\arraystretch}{1.5}
    \caption{Some loss functions and their asymptotic risk minimisers in a binary classification setting. We show the predicted probability for class 1,  $\model^*_1(x)$, but $\model^*_2 (x) = 1 - \model^* _1 (x)$. For short notation we use $\gamma \in \text{argmin}_{q \in [0, 1]} \loss (q, 1)$ and $\gamma' \in \text{argmax}_{q \in [0,1]} \loss (q, 1)$.
    }
    \label{tab:pop_minimisers}
    \begin{tabular}{ccc}
         Loss 
         & $\model ^{*}_1 (x)$ & $\sum_{k=1}^{K} \loss(\model(x), k)$ \\ \toprule
         
         CCE 
         & $\prob(Y=1\mid X=x)$ &  $- \sum_{k=1}^{2} \log \model_k (x)$ \\
         
         
         MAE 
         & $\mathbb{I}_{\prob(Y=1\mid X=x) \geq \frac{1}{2}}$ & 2 
         \\
         
         Sigmoid 
         & $\mathbb{I}_{\prob(Y=1\mid X=x) \geq \frac{1}{2}}$ & 1 \\
         
         
         Symmetric 
         & $\gamma \mathbb{I}_{\prob(Y=1\mid X=x) \geq \frac{1}{2}} + \gamma' \mathbb{I}_{\prob(Y=1\mid X=x) < \frac{1}{2}}$ & C\\
         \bottomrule
    \end{tabular}
    \linebreak
    }
\end{table}

\subsection{CCE Loss} 

\paragraph{Loss function:} $\loss(\model(x), y) = - \sum_{k=1}^{2} \onehoty_k \log \model_k (x) = - \log \model_{y} (x)$

\paragraph{Symmetry check:} 
\begin{align*}
    \sum_{k=1}^{K} \loss(\model(x), k) & = - \sum_{k=1}^{2}\sum_{l=1}^{2}e_l^{(k)} \log \model_k (x) = - \sum_{k=1}^{2} \log \model_k (x)
\end{align*}
Not symmetric.

\paragraph{Risk minimiser:} $\model^*_1(x) = \truemodel _1(x)$

\textit{\\Derivation:}
\begin{align*}
  & \mathcal{J}(\model (x)) = -  \truemodel _1(x) \log \model_1(x) - \truemodel _2(x) \log \model_2(x) \\ & \qquad \quad \;\;  = - \truemodel _1(x) \log \model_1(x) - (1-\truemodel _1(x)) \log (1-\model_1(x))
\\[10pt] &
\Rightarrow
    \frac{\partial \mathcal{J}(\model (x))}{\partial \model_1 (x)} =  - \truemodel _1(x) \frac{1}{\model_1(x)} + (1-\truemodel _1(x))\frac{1}{1-\model_1(x)}
\\[10pt] &
\Rightarrow
    \frac{\partial \mathcal{J}(\model (x))}{\partial \model_1 (x)} = 0 \Rightarrow \model^*_1(x) = \truemodel _1(x) 
\end{align*}


%
%


%

\subsection{MAE} 

\paragraph{Loss:} $\loss(\model(x), y) = \sum_{k=1}^{2}|\onehoty_k  - \model_k(x)|$

\paragraph{Symmetry check:} 
\begin{align*}
    \sum_{k=1}^{K} \loss(\model(x), k) & = \sum_{k=1}^{2} \sum_{l=1}^{2} |e_l^{(k)} - \model_{j}(x)| = 2\cdot 2 - 2 = 2
\end{align*}
Symmetric. 

\paragraph{Risk minimiser: } $f^*_1(x) = \mathbb{I}_{\prob(Y=1\mid X=x) \geq \frac{1}{2}}$

\textit{\\Derivation: }
\begin{align*}
     \mathcal{J}(\model (x)) &= \truemodel _1(x) (|1-\model_1(x)| + |-\model_2(x)|) + \truemodel _2(x) (|-\model_1(x)| + |1-\model_2(x)|) \\ & =  2\truemodel _1(x) (1-\model_1(x)) + 2(1-\truemodel _1(x)) \model_1(x)
\end{align*}
If $\truemodel _1(x) \geq \frac{1}{2}$, minimum at $\model_1(x) = 1$, otherwise minimum at $\model_1(x) = 0$. 
We can formulate this as $\model^*_1(x) =\mathbb{I}_{\prob(Y=1\mid X=x) \geq \frac{1}{2}}$. 

\subsection{Sigmoid Loss} 

\paragraph{Loss function: } $\loss(\model (x), y) = \onehoty_1 \frac{1}{1+ e^{\model_1 (x)}} + \onehoty_2 \frac{e^{\model_1(x)}}{1+ e^{\model_1 (x)}}$

\paragraph{Symmetry check:} 
\begin{align*}
    \sum_{k=1}^2 \loss(\model(x), k) & = \loss(\model(x), 1) + \loss(\model(x), 2) = \frac{1}{1 + e^{\model_1 (x)}} + \frac{e^{\model_1 (x)}}{1 + e^{\model_1 (x) }} = 1
\end{align*}
Symmetric.

\paragraph{Risk minimiser:} $f^*_1(x) = \mathbb{I}_{\prob(Y=1\mid X=x) \geq \frac{1}{2}}$

\textit{\\Derivation: }

\begin{align*}
     \mathcal{J} (\model(x)) & = \truemodel _1(x) \frac{1}{1+ e^{\model_1 (x)}} + \truemodel _2(x) \frac{ e^{\model_1 (x)}}{1+ e^{\model_1 (x)}} \\& = \truemodel _1(x) \frac{1}{1+ e^{\model_1 (x)}} + (1 - \truemodel _1(x)) (1 - \frac{1}{1+ e^{\model_1 (x)}})
\end{align*}
\if 
$\frac{d $\mathcal{J} (\model(x)) }{d \model (x)} = \frac{e^{\model(x)}}{(1 + e^{\model(x)})^2} (1-2\truemodel _1(x))  $
$\frac{d \mathcal{J} (\model(x))}{d \model(x)} = 0 \Rightarrow 
$e^{\model}(x) = 0 \Rightarrow \model(x) \rightarrow -\infty$
$\frac{1}{(1 + e^{\model}(x))}^2 = 0 \Rightarrow \model(x) \rightarrow \infty$.
\fi
If $\truemodel _1(x) \geq \frac{1}{2}$, minimum at $\model_1(x) \rightarrow \infty$, otherwise minimum at $\model_1(x) \rightarrow -\infty$.
Since $\model_1(x) \in [0, 1]$ and $\frac{1}{1 + e^{\model_1 (x)}}$ is a (monotonic) decreasing function in $\model_1(x)$, we have a minimum at $\model_1(x) = 1$ for $\truemodel _1(x) \geq \frac{1}{2}$ and at $\model_1(x) = 0$ for $\truemodel _1(x) < \frac{1}{2}$. So, $\model^*_1 (x) =\mathbb{I}_{\prob(Y=1\mid X=x) \geq \frac{1}{2}}$.

\subsection{General Symmetric Loss} 

\paragraph{Loss function:} $\loss (f(x), y)$

\paragraph{Symmetry check: } 
\begin{align*}
    \sum_{k=1}^{2} \loss (f(x), k) = C
\end{align*}
Symmetric by definition.

\paragraph{Risk minimiser: } $\model^*_1(x) = \gamma \mathbb{I}_{\prob(Y=1\mid X=x) \geq \frac{1}{2}} + \gamma' \mathbb{I}_{\prob(Y=1\mid X=x) < \frac{1}{2}}$

\textit{\\Derivation:}
\begin{align*}
   \mathcal{J}(\model (x)) & = \truemodel _1(x) \loss (\model(x), 1) + \truemodel _2(x) \loss (\model(x), 2) \\ & = \truemodel _1(x) \loss (\model(x), 1) + (1 - \truemodel _1(x)) (C - \loss(\model(x), 1)) \\ & = \loss(\model(x), 1)(2\truemodel _1(x)-1) + C(1-\truemodel _1(x)) 
\end{align*}
If $\truemodel _1(x) \geq \frac{1}{2}$, minima at $\text{argmin}_{q \in [0, 1]} \loss(q, 1)$, otherwise minima at $\text{argmax}_{q \in [0,1]} \loss(q, 1)$. Hence, minima is found at 
\begin{align*}
  \model^*_1(x) = \gamma \mathbb{I}_{\prob(Y=1\mid X=x) \geq \frac{1}{2}} + \gamma' \mathbb{I}_{\prob(Y=1\mid X=x) < \frac{1}{2}}  
\end{align*}
with $\gamma \in \text{argmin}_{q \in [0, 1]} \loss(q, 1)$ and $\gamma ' \in  \text{argmax}_{q \in [0, 1]} \loss(q, 1)$. Since we assume $\model_1 (x) \in [0, 1]$, we must also have $\gamma, \gamma' \in [0, 1]$. 

\section{Complete proofs}
\label{sec:complete_proofs}

Complete proofs of the statements, lemmas and propositions from the main paper follow.\\

\noindent\textbf{\Cref{lem:func_of_g}.} \textit{\space The conditional probability vector $\noisymodel(x)$ can be written as function of $\truemodel (x)$ according to}
    \begin{equation*}
        \noisymodel(x) = 
        \left(1-\frac{\omega(x)K}{K-1}\right)\truemodel(x) + \frac{\omega(x)}{K-1} \cdot \mathbf{1}_K
    \end{equation*}
    \textit{where $\mathbf{1}_K$ is the vector of ones with length $K$. Moreover, it holds for any two classes $i, j \in \yspace$ that $\noisymodel_i (x) > \noisymodel_j (x)$ if and only if $\truemodel_i (x) > \truemodel_j (x)$.}

\begin{proof}
    For the first part of the proof, note that for any $i \in \yspace$, $\sum_{k \neq i} \truemodel_k (x) = 1 - \truemodel_i (x)$, then
    \begin{align*}
         \noisymodel_{i}(x) & = 
         \sum_{k=1}^{K} \prob(\Noisyy = i \mid Y=k, X=x) \truemodel_k(x)
         \\ & = (1-\omega(x))\truemodel_i (x) + \frac{\omega(x)}{K-1} \sum_{k\neq i} \truemodel_k (x)
        \\ & = 
        (1-\omega(x))\truemodel_i (x) + \frac{\omega(x)}{K-1} (1-\truemodel_i (x) )
        \\ & = 
        \left(1-\frac{\omega(x)K}{K-1}\right)\truemodel_i(x) + \frac{\omega(x)}{K-1}.
    \end{align*}
    Since this holds for all $i \in \yspace$, we can write the result on vector form according to
    \begin{align*}
        \noisymodel(x) & = 
        \left(1-\frac{\omega(x)K}{K-1} \right)\truemodel(x) + \frac{\omega(x)}{K-1} \cdot \mathbf{1}_K 
    \end{align*}
    where $\mathbf{1}_K$ is a vector of ones with length $K$. This finishes the first part of the proof.
    
    For the second part of the proof. Take any two classes $i, j \in \yspace$ and assume $\truemodel_i (x) > \truemodel_j (x)$. Then, from the equation derived above and since $0 \leq \omega(x) < \frac{K-1}{K}$, we get
    \begin{align*}
        \noisymodel_{i}(x) & = 
        \left(1-\frac{\omega(x)K}{K-1}\right)\truemodel_i(x) + \frac{\omega(x)}{K-1} \\ & >  \left(1-\frac{\omega(x)K}{K-1}\right)\truemodel_j(x) + \frac{\omega(x)}{K-1} \\ &= \noisymodel_{j}(x).
    \end{align*}
    Hence, $\truemodel_i (x) > \truemodel_j (x)$ implies $\noisymodel_i (x) > \noisymodel_j (x)$.
    
    Next, assume $\noisymodel_i (x) > \noisymodel_j (x)$. Then by rearranging the same equation, we obtain the following
    \begin{align*}
        \truemodel_i (x) &= \frac{\noisymodel_i (x) - \frac{\omega(x)}{K-1}}{1 - \frac{\omega(x)K}{K-1}} \\ & > \frac{\noisymodel_j (x) - \frac{\omega(x)}{K-1}}{1 - \frac{\omega(x)K}{K-1}} \\ & = \truemodel_j (x).
    \end{align*}
    Hence, $\noisymodel_i (x) > \noisymodel_j (x)$ implies $\truemodel_i (x) > \truemodel_j (x)$. Therefore, $\noisymodel_i (x) > \noisymodel_j (x)$ if and only if $\truemodel_i (x) > \truemodel_j (x)$.
\end{proof}

\begin{lemma}
    For simple non-uniform label noise, $\noisymodel(x) = \truemodel(x)$ if and only if $\omega(x) = 0$ or $\truemodel(x) = \frac{1}{K} \cdot \mathbf{1}_K$.
    \label{lem:diff_g}
\end{lemma}

\begin{proof}
    First, we show that $\omega(x) = 0$ or $\truemodel(x) = \frac{1}{K} \cdot \mathbf{1}_K$ implies $\truemodel(x) = \noisymodel(x)$. Assume $\omega(x) = 0$, then from \cref{lem:func_of_g}:
    \begin{align*}
        \noisymodel(x) = 
        \left(1-\frac{0 \cdot K}{K-1}\right)\truemodel(x) + \frac{0}{K-1} \cdot \mathbf{1}_K = \truemodel(x)
    \end{align*}
    Similarly, if $\truemodel(x) = \frac{1}{K} \cdot \mathbf{1}_K$, then 
    \begin{align*}
        \noisymodel(x) = 
        \left(1-\frac{\omega(x) K}{K-1}\right)\frac{1}{K} \cdot \mathbf{1}_K + \frac{\omega(x)}{K-1} \cdot \mathbf{1}_K = \frac{1}{K} \cdot \mathbf{1}_K
    \end{align*}
    i.e. $\noisymodel(x) = \truemodel(x)$ and so, $\omega(x) = 0$ or $\truemodel(x) = \frac{1}{K} \cdot \mathbf{1}_K$ implies $\truemodel(x) = \noisymodel(x)$.
    
    Next, we show that $\noisymodel(x) = \truemodel (x)$ implies $\omega(x) = 0$ or $\truemodel(x) = \frac{1}{K} \cdot \mathbf{1}_K$. Assume $\noisymodel(x) = \truemodel(x)$, then from \cref{lem:func_of_g}
    \begin{align*}
        \noisymodel(x) = 
        \left(1-\frac{\omega(x) K}{K-1}\right)\truemodel(x) + \frac{\omega(x)}{K-1} \cdot \mathbf{1}_K = \truemodel (x)
    \end{align*}
    which implies $\omega(x) = 0$ or
    \begin{align*}
        K\truemodel(x) - \mathbf{1}_K = 0
    \end{align*}
    or, equivalently, $\truemodel(x) = \frac{1}{K} \cdot \mathbf{1}_K$. Therefore, if $\noisymodel(x) = \truemodel(x)$, then $\omega(x) = 0$ or $\truemodel(x) = \frac{1}{K} \cdot \mathbf{1}_K$. Hence, we have shown that $\noisymodel(x) = \truemodel(x)$ if and only if $\omega(x) = 0$ or $\truemodel(x) = \frac{1}{K} \cdot \mathbf{1}_K$.
\end{proof}

\newpage
\noindent\textbf{\Cref{prop:decision_boundaries}.} \textit{\space Assume that the prediction is taken as the most probable class, then $\noisymodel$ has the same decision boundaries as $\truemodel$.}

\begin{proof}
     To show that $\noisymodel$ and $\truemodel$ have the same decision boundaries, we will show that $\noisymodel_i(x) = \noisymodel_j(x)$ if and only if  $\truemodel_i(x) = \truemodel_j(x)$ for all $x \in \xspace$ for which this holds and for any two classes $i, j \in \yspace$. In addition, \cref{lem:func_of_g} states that classes are not flipped, as $\noisymodel_i(x) > \noisymodel_j(x)$ if and only if $\truemodel_i(x) > \truemodel_j(x)$.
     
     We first show that if $\truemodel_i(x) = \truemodel_j(x)$, then $\noisymodel_i(x) = \noisymodel_j(x)$. For $i, j \in \yspace, \; i \neq j$, assume $\truemodel_{i}(x)=\truemodel_{j}(x)$. Then, for any $x$ for which this holds (i.e. any $x$ along the decision boundary between classes $i$ and $j$)
    \begin{align*}
        \noisymodel_{j}(x) & = \left(1-\frac{\omega(x)K}{K-1}\right)\truemodel_j(x) + \frac{\omega(x)}{K-1} \\ & = \left(1-\frac{\omega(x)K}{K-1}\right)\truemodel_i(x) + \frac{\omega(x)}{K-1} \\ &= \noisymodel_i (x).
    \end{align*}
    This follows directly from \cref{lem:func_of_g} and since $\truemodel_i (x) = \truemodel_j (x)$.
    
    Next, we show that $\noisymodel_i(x) = \noisymodel_j(x)$ implies $\truemodel_i(x) = \truemodel_j(x)$. Assume $\noisymodel_{i}(x)=\noisymodel_{j}(x)$ for some $i, j \in \yspace, \; i \neq j$. From \cref{lem:func_of_g}, we get 
    \begin{align*}
        \truemodel_i (x) & = \frac{\noisymodel_i(x) - \frac{\omega(x)}{K-1}}{1-\frac{\omega(x) K}{K-1}}
    \end{align*}
    and therefore,
    \begin{align*}
        \truemodel_i (x) & = \frac{\noisymodel_i(x) - \frac{\omega(x)}{K-1}}{1-\frac{\omega(x) K}{K-1}} \\ & = 
        \frac{\noisymodel_j(x) - \frac{\omega(x)}{K-1}}{1-\frac{\omega(x) K}{K-1}}\\ & = \truemodel_j (x),
    \end{align*}
    i.e. if $\noisymodel_i(x) = \noisymodel_j(x)$, then $\truemodel_i(x) = \truemodel_j(x)$.
    
    This shows that for any $x \in \xspace$ it holds that $\noisymodel_i (x) = \noisymodel_j (x)$ if and only if $\truemodel_i (x) = \truemodel_j (x)$. Note that the argument easily extends to equality between several classes, i.e. $\noisymodel_{i}(x)=\noisymodel_{j}(x)=\dots=\noisymodel_{l}(x)$ if and only if $\truemodel_{i}(x)=\truemodel_{j}(x)=\dots=\truemodel_{l}(x)$, since we still consider pairwise equality. In turn, this implies, together with \cref{lem:func_of_g}, that $\noisymodel$ and $\truemodel$ share decision boundaries.
\end{proof}

\noindent\textbf{\Cref{prop:conditional_entropy}.} \textit{\space The average conditional entropy of $\noisymodel$ is higher than that of $\truemodel$, that is $\expectation_X[\entropy{\noisymodel(X)}] > \expectation_X[\entropy{\truemodel(X)}]$.}

\begin{proof}
To show that $\noisymodel$ has a higher average conditional entropy than $\truemodel$, we first note that the conditional entropy is a strictly concave function. This follows from the fact that the entropy is a sum of strictly concave functions $h(z) = - z\log (z)$.  %
Hence, for any two distinct vectors $\model^1(x), \model^2(x)$
\begin{align*}
    \entropy{\model^{\lambda}(x)} > (1-\lambda)\entropy{\model^1(x)} + \lambda \entropy{\model^2(x)}, \quad \lambda \in (0, 1)
\end{align*}
with $f^\lambda (x) = (1-\lambda) f^1(x) + \lambda f^2 (x)$. The unique maximum of the conditional entropy is found at $u = \frac{1}{K} \cdot \mathbf{1}_K$, the vector of uniform probability.

Notice that for $\lambda = \frac{\omega(x) K}{K-1}$, we can use \cref{lem:func_of_g} to write $\noisymodel(x)$ as a linear combination of the form
\begin{align*}
    \noisymodel(x) = (1-\lambda) \truemodel(x) + \lambda \frac{1}{K} \cdot \mathbf{1}_K
\end{align*}
Let $f^1 (x) = \truemodel (x) $, $f^2 (x) = u$ (the uniform vector) and $\lambda = \frac{\omega(x) K}{K-1}$. For $\omega(x) > 0$ and $\truemodel (x) \neq u$, we have
\begin{align*}
    \entropy{\noisymodel(x)} & >   (1-\lambda)\entropy{\truemodel(x)} + \lambda \entropy{u} \; \\ & > (1-\lambda)\entropy{\truemodel(x)} + \lambda \entropy{\truemodel (x)} \\ & = \entropy{\truemodel (x)}.
\end{align*}
The last inequality follows since the maximum at $u$ is unique and since $\lambda > 0$. Hence, for $\truemodel (x) \neq u$ and $\omega(x) > 0$, it holds that $\entropy{\noisymodel(x)} > \entropy{\truemodel (x)}$. For $\truemodel (x) = u$ or if $\omega(x) = 0$, it follows from \cref{lem:diff_g} that $\noisymodel(x) =  \truemodel(x)$ and, therefore, $\entropy{\noisymodel(x)} = \entropy{\truemodel (x)}$. 

Let $\xspace_{1} = \{x \in \xspace; \; \noisymodel(x) = \truemodel(x)\}$ and $\xspace_{2} = \{x \in \xspace; \; \noisymodel(x) \neq \truemodel(x)\}$. Moreover, let $\mu_X$ be the marginal probability distribution of $X$. Since $\prob\left(\{\omega(X) > 0\} \cap \{\truemodel(X) \neq u\} \right) > 0$ implies $\prob(X \in \xspace_2) > 0$  (\cref{lem:diff_g}), we have
\begin{align*}
    \expectation_X [\entropy{\noisymodel(X)} ]
    & =  \int_{\xspace_1} \entropy{\noisymodel(x)} \mu_X(dx) + \int_{\xspace_2} \entropy{\noisymodel(x)} \mu_X(dx)
    \\ & > \int_{\xspace_1} \entropy{\truemodel(x)}  \mu_X(dx) + \int_{\xspace_2} \entropy{\truemodel(x)} \mu_X(dx) \\ &= 
    \expectation_X [\entropy{\truemodel (X)}]
\end{align*}
meaning that $\noisymodel$ has a higher conditional entropy than $\truemodel$ on average over $\xspace$.

\end{proof}

\begin{lemma}
    Let $M \in \argmax{k \in \yspace} \noisymodel_k (x)$, then $\truemodel_M (x) \geq \noisymodel_M (x)$ with equality if and only if $\truemodel (x) = \noisymodel (x)$.
    \label{lem:argmax_lemma}
\end{lemma}

\begin{proof}

We first show that for any  $M \in \argmax{k} \noisymodel_k (x)$, it holds that $\truemodel_M (x) \geq \noisymodel_M (x)$. Note that it follows from \cref{lem:func_of_g} and \cref{prop:decision_boundaries} that 
\begin{align*}
    \argmax{k} \truemodel_k (x) = \argmax{k} \noisymodel_k (x)
\end{align*}
such that if $M \in \argmax{k} \noisymodel_k (x)$, then $M \in \argmax{k} \truemodel_k (x)$. Note also that it must hold that $\noisymodel_M(x), \truemodel_M (x) \geq \frac{1}{K}$, since $\noisymodel(x)$ and $\truemodel(x)$ are both vectors of norm 1. From \cref{lem:func_of_g}, we have
\begin{align*}
    \noisymodel_M (x) & = \left(1-\frac{\omega(x)K}{K-1}\right)\truemodel_M(x) + \frac{\omega(x)}{K-1} \\ &= (1-K\truemodel_M(x)) \frac{\omega(x)}{K-1} + \truemodel_M(x)
\end{align*}
which is a linear function in $\omega(x)$. Since $\truemodel_M(x) \geq \frac{1}{K}$, $\noisymodel_M (x)$ decreases with $\omega(x)$ and an upper bound is found at $\omega(x) = 0$
\begin{align*}
    \noisymodel_M (x) \leq (1 - K\truemodel_M (x)) \frac{\omega(x)}{K-1} + \truemodel_M(x) \biggr\rvert_{\omega(x) = 0} = \truemodel_M (x),
\end{align*}
which finishes the first part of the proof.

Next, we show that $\truemodel_M (x) = \noisymodel_M (x)$, for any $M \in  \argmax{k} \truemodel_k (x)$, if and only if $\truemodel (x) = \noisymodel (x)$. First, it follows directly that if $\truemodel (x) = \noisymodel (x)$ then $\truemodel_k(x) = \noisymodel_k(x) \sforall k \in \yspace$ and therefore, $\truemodel_M (x) = \noisymodel_M (x)$. Second, to see that $\truemodel_M (x) = \noisymodel_M (x)$ implies $\truemodel (x) = \noisymodel(x)$, we use \cref{lem:func_of_g} with $\truemodel_M (x) = \noisymodel_M (x)$, to get
\begin{align*}
    \noisymodel_M (x) = \left(1-\frac{\omega(x)K}{K-1}\right)\noisymodel_M(x) + \frac{\omega(x)}{K-1}
\end{align*}
which implies $\omega(x)=0$ or 
\begin{align*}
    (K \noisymodel_M (x) - 1) = 0,
\end{align*}
i.e. $\noisymodel_M(x) = \frac{1}{K}$. For $\omega(x)=0$, it follows directly from \cref{lem:func_of_g} that $\truemodel (x) = \noisymodel (x)$. In the second case, notice that $\noisymodel_M(x) = \frac{1}{K}$ must imply that $\noisymodel_k = \frac{1}{K} \sforall k \in \yspace$, since $M \in \argmax{k} \noisymodel_{k} (x)$ and $\ \|\noisymodel_{k} (x)\|_1 = 1$. Then, from \cref{lem:func_of_g}, we get
\begin{align*}
    \truemodel (x) & = 
    \frac{(\frac{1}{K} - \frac{\omega(x)}{K-1}) \cdot \mathbf{1}_K}{1 - \frac{\omega(x)K}{K-1}} 
    \\ & = 
    \frac{\frac{1}{K}(1 - \frac{\omega(x)K)}{K-1}}{(\frac{1}{K} - \frac{\omega(x)K}{K-1})} 
    \\ & = \frac{1}{K} \cdot \mathbf{1}_K
\end{align*}
which means that $\truemodel(x) = \noisymodel(x)$. Therefore, $\truemodel_M (x) \geq \noisymodel_M (x)$ with equality if and only if $\truemodel (x) = \noisymodel (x)$.

\end{proof}

\noindent\textbf{\Cref{prop:noisyy_calibration}.} \textit{\space The vector of conditional probabilities $\noisymodel (X)$ is not calibrated with respect to the distribution $\prob(Y \mid X)$ over clean labels.}
\newcommand{\wpone}{\text{almost surely}}
\newcommand{\I}[1]{\mathbbm{1}(#1)}

\begin{proof}
Let $Z = \truemodel (X)$ and $\tilde{Z} = \noisymodel (X)$.
For $\noisymodel(X)$ to be calibrated we require,
\begin{align*}
    \prob(Y=k \mid \tilde{Z}) &= \tilde{Z}_k, \sforall k \in \yspace
\end{align*}
almost surely. Hence, to show that $\noisymodel$ is not calibrated, it is enough to show that $\prob(Y = k \mid \tilde{Z}) \neq \tilde{Z}_k$, with probability larger than 0, for any $k \in \yspace$. 
Let
\begin{align*}
    M = \inf\argmax{k} \tilde{Z}_k,
\end{align*}
where the infimum is taken just to select a unique index in the case when the maximising argument is not unique. 
We know from \cref{lem:func_of_g} and \cref{prop:decision_boundaries} that
\begin{align*}
    M = \inf\argmax{k} {Z}_k 
\end{align*}
meaning that $M$ is a $\sigma(\tilde{Z})\cap\sigma(Z)$-measurable random variable. From \cref{lem:argmax_lemma} we know that
\begin{align*}
    Z_M \geq \tilde Z_{M}
\end{align*}
with equality if and only if $Z = \tilde Z$ \wpone.
From the tower property of conditional expectation we get
\begin{align*}
    \prob(Y=M \mid \tilde{Z})
    = \expectation_{Z \mid \tilde{Z}}[\expectation_{Y \mid Z, \tilde{Z}}[\mathbb{I}_{Y=M}]]
    = \expectation_{Z \mid \tilde Z}[ Z_M]
\end{align*}
where the second equality follows from the fact that $Y$ is conditionally independent of $\tilde Z$ given $Z$, which in turn follows from the properties of the selected noise model,
and that
\begin{align*}
    \prob(Y=M \mid Z) = Z_M
\end{align*}
by definition.
Hence, since $\prob (Z \neq \tilde Z)>0$ by assumption, there is a set with measure strictly larger than zero on which
\begin{align*}
    \prob(Y=M \mid \tilde{Z}) > \tilde Z_{M}
\end{align*}
and
$\noisymodel(X)$ is not calibrated for $\prob(Y \mid X)$.

\end{proof}

\newpage

\noindent\textbf{\Cref{prop:robustness_strictly_proper}.} \textit{\space Robust loss functions (\cref{def:robustness_condition}) are not strictly proper.}
\begin{proof}
    Assume that the loss function $\loss$ is robust according to \cref{def:robustness_condition}, i.e. $\loss \in \losspace{R}$. 
    Since $\loss$ is robust, it holds for all $\model^* \in \argmin{\model} \risk (\model)$ that 
    \begin{align*}
        \noisyrisk (\model^*) \leq \noisyrisk (\model), 
        \quad \sforall \model \in \modelspace{},
    \end{align*}
    with equality only if $\model$ is also in the set of asymptotic risk minimisers.
    
    Assume now that $\loss$ is strictly proper. Then,
    \begin{align*}
       \argmin{\model} \risk (\model) = \{\truemodel\}.
    \end{align*}
    In parallel, strictly properness implies
    \begin{align*}
         \argmin{\model} \noisyrisk (\model) = \{\noisymodel\}
    \end{align*}
    and therefore,
    \begin{align*}
        \noisyrisk (\noisymodel) \leq \noisyrisk (\truemodel),
    \end{align*}
    with equality only in the noise-free case where $\noisymodel=\truemodel$. Otherwise, this is a contradiction. As a result, $\loss \notin \losspace{SP}$. Since $\loss$ is an arbitrary robust loss function, we have shown that
    \begin{align*}
        \losspace{R} \cap \losspace{SP} = \emptyset.
    \end{align*}
\end{proof}

\noindent\textbf{\Cref{prop:robustness_calibrated}.} \textit{\space Symmetric loss functions (\cref{def:symmetric_loss}) are not calibration-based strictly proper.} 
\begin{proof}
For $\losspace{S} \cap \losspace{CSP} = \emptyset$ to hold, we require that every symmetric loss function has at least one asymptotic risk minimiser that is not calibrated. That is, every $\loss \in \losspace{S}$, has at least one risk minimiser $\model^* \notin \modelspace{C}$, for at least one conditional distribution $\prob(Y \mid X)$ over the target, $Y$, and one probability distribution, $\mu_X$, over the input, $X$. 

Consider the binary case and a symmetric loss function with (point-wise) asymptotic risk minimisers defined in \cref{eq:symmetric_loss_minima}. For an arbitrary symmetric loss function $\loss \in \losspace{S}$ and if at least one risk minimiser exists, one of the following holds for the parameters $\gamma, \gamma'$ in \cref{eq:symmetric_loss_minima}:

\begin{enumerate}[label=(\roman*)]
    \item \label{itm:first_option} Both $\gamma$ and $\gamma'$ are unique.
    \item \label{itm:second_option} At least one of $\gamma$, $\gamma'$ is not unique.
\end{enumerate}
Following the definitions of $\gamma$ and $\gamma'$, and since $\loss$ is solely a function of a probability vector $q$ and a label $y$, we know that $\gamma$ and $\gamma'$ are both independent of the given input $x$. Hence, for \cref{itm:first_option}, the full risk minimiser will take the form
\begin{align*}
    \model^{*}_{1}(x) = \begin{cases}
    \gamma,\quad \text{if} \;  \prob(Y=1 \mid X=x) \geq \frac{1}{2} \\
    \gamma', \quad \! \text{otherwise}
    \end{cases}
    \label{eq:constant_minimiser}
\end{align*}
$\forall x \in \xspace$ and where $\gamma, \gamma'$ are constants.

Let $\xspace_1 = \{x \in \xspace; \; \prob(Y=1 \mid X=x) \geq \frac{1}{2}\}$ and $\xspace_2 = \{x \in \xspace; \; \prob(Y=1 \mid X=x) < \frac{1}{2} \}$. For $\model^*$ to be calibrated, we require 
\begin{align*}
    &\prob(Y = 1 \mid \model^{*}_{1}(X) = \gamma) = \prob(Y = 1 \mid X \in \xspace_1) = \gamma, \\
    &\prob(Y = 1 \mid \model^{*}_{1}(X) = \gamma') = \prob(Y = 1 \mid X \in \xspace_2) = \gamma'.
\end{align*}
This is true if $\gamma$ (coincidentally) matches the average probability of $Y=1$ over $\xspace_1$ and $\gamma'$ is equal to the average probability of $Y=1$ over $\xspace_2$ according to
\begin{align*}
    &\prob(Y=1 \mid X \in \xspace_1) = \frac{\int_{\xspace_1} \prob(Y=1 \mid X=x)\mu_X(dx)}{\int_{\xspace_1} \mu_X(dx)} = \gamma, \\
    &\prob(Y=1 \mid X \in \xspace_2) = \frac{\int_{\xspace_2} \prob(Y=1 \mid X=x)\mu_X(dx)}{\int_{\xspace_2} \mu_X(dx)}  = \gamma'.
\end{align*}
To see that the model $\model^*$ is not calibrated in general, assume that there is a conditional probability $\prob(Y \mid X)$ and marginal distribution $\mu_X$ for which the model is calibrated, i.e. the equations above hold. Then, consider any other data distribution for which $\mu_X$ is the same but where the conditional probability over $Y$ can be described by
\begin{align*}
    \prob'(Y=1 \mid X=x) = (1-2\alpha)\prob(Y=1 \mid X=x) + \alpha, \quad 0 < \alpha < \frac{1}{2},
\end{align*}
such that $\xspace_1$ and $\xspace_2$, and consequently $f^*$, remain the same.
For this data distribution and for any suitable solution $\gamma \in [\frac{1}{2}, 1], \; \gamma' \in [0, \frac{1}{2})$, we find that
\begin{align*}
    \prob'(Y=1 \mid X \in \xspace_1) & = \frac{\int_{\xspace_1} \prob'(Y=1 \mid X = x)\mu_X(dx)}{\int_{\xspace_1}\mu_X(dx)} \\ &= \frac{\int_{\xspace_1} \left((1-2\alpha)\prob(Y=1 \mid X=x) + \alpha\right)\mu_X(dx)}{\int_{\xspace_1}\mu_X(dx)} \\ & = (1-2\alpha)\gamma + \alpha \leq \gamma
\end{align*}
with equality only if $\gamma=\frac{1}{2}$. Similarly,
\begin{align*}
    \prob'(Y=1 \mid X \in \xspace_2) & = \frac{\int_{\xspace_2} \prob'(Y=1 \mid X = x)\mu_X(dx)}{\int_{\xspace_2} \mu_X(dx)} \\ & = (1-2\alpha)\gamma' + \alpha > \gamma'.
\end{align*}
As a result, $\model^*$ is not calibrated for the conditional distribution $\prob'(Y\mid X)$. Consequently, we can conclude that there exists a conditional distribution $\prob(Y \mid X)$ and marginal $\mu_X$ for which the asymptotic risk minimiser $\model^*$ of $\loss$ is not calibrated.

Now, assume \cref{itm:second_option} holds, i.e. the loss $\loss$ has two or more (point-wise) asymptotic risk minimers. Since every combination $(\gamma, \gamma')$, from the set of point-wise asymptotic risk minimisers of $\loss$, is equally viable for every $x \in \xspace$ (they are all independent of $x$ and minimise the point-wise risk), at least one risk minimiser can be formulated according to the unique risk minimiser in case \cref{itm:first_option}. Hence, there exist at least one risk minimiser $\model^*$ for which it holds that $\model^* \notin \modelspace{C}$ for some $\prob(Y \mid X)$ and $\mu_X$. 

With the arguments put forth, it holds for all symmetric loss functions that for some conditional distribution $\prob(Y\mid X)$ and marginal $\mu_X$, there exists an asymptotic risk minimiser that is not calibrated. It follows that symmetric loss functions are not calibration-based strictly proper, i.e. $\losspace{S} \cap \losspace{CSP} = \emptyset$.
\end{proof}

\section{Experimental details}
The empirical experiments were performed using Python and the Pytorch deep learning library \parencite{Paszke2019}. Experimental details follow in this section.

\subsection{Simple Noise Example}
For the simple noise example in \cref{fig:intro_example}, the models used were fully connected neural networks with one hidden layer of 50 hidden units and with ReLU activation. The models were trained on two-dimensional circle data\footnote{\url{https://scikit-learn.org/stable/modules/generated/sklearn.datasets.make_circles.html}} with 5,000 observations, generated with the scikit-learn library \parencite{Pedregosa2011}. For the noisy data set, labels were flipped with a uniform flip probability of $\omega=0.2$. Both models were trained with categorical cross-entropy loss and ADAM optimization \parencite{Kingma2015}. We used a constant learning rate of 0.1 and a batch size of 100. A separate validation set of size 1,000 was used for early stopping, where the training was stopped if the current validation loss value exceeded the minimum achieved with more than $10\%$. From that, the model with the smallest validation loss was selected. The models were evaluated in terms of accuracy on a separately generated, clean, test data set of 1,000 samples. In addition, plots of the predicted class 1 probability for each model were generated on a grid of range $(-1.5, 1.5)$ in both dimensions.

\subsection{Robust Loss Functions and Overfitting}
For the empirical evaluation of robustness against overfitting, all models used were fully connected neural networks with one hidden layer of 500 hidden units and LeakyReLU activation. The models were trained with ADAM optimisation \parencite{Kingma2015} with a constant learning rate of 0.005 and a batch size of 100. The hyperparameters were selected such that the general trends of the training dynamics could be observed. 

The data used for training was the MNIST training data set \parencite{Lecun1998} where 50,000 of the data points were used for training and 10,000 for validation. The images were flattened prior to training. Two sets of data with symmetric label noise were created by random flipping of labels. For the first noisy set, a flip probability of $\omega=0.3$ was used and for the second set, we used $\omega=0.5$. The original data set was assumed to be noise free, corresponding to a flip probability of $\omega=0.0$. 

The models trained with mean absolute error (we use the built-in L1Loss in Pytorch with mean reduction, which in our framework corresponds to training with MAE $\cdot \frac{1}{K}$) were trained for 5,000 epochs and evaluated every 10$^{th}$ epoch. The models were initialised randomly or from the weights obtained by a separate model trained with categorical cross-entropy (CCE) loss on the corresponding clean or noisy data set. The models trained with CCE loss, both for the purpose of pre-training and for the purpose of separate evaluation, were trained for 500 epochs and evaluated every epoch, if relevant.

The models were evaluated both on the data set on which they were trained and on the separate, presumably clean, MNIST test data set of 10,000 observations. We evaluated all models in terms of accuracy. In addition, we compared the training loss (MAE) on the respective training sets for the models trained with MAE, with and without pre-training.

\section{Information-theoretic loss function}
Apart from the symmetric loss functions identified in \cite{Ghosh2015, Ghosh2017}, the information-theoretic loss function proposed by \parencite{Xu2019} is robust according to \cref{def:robustness_condition} under symmetric label noise. We will show that while this loss function is robust, it does not recover $\truemodel$ and it is not calibration-based strictly proper. The information-theoretic loss function is based on Determinant-based Mutual Information (DMI) and is defined as
\begin{align*}
    \loss(\model(X), Y) = -\log |\text{det}(\prob(\hat{Y}, Y)|, \quad \hat{Y} \sim \model(X).
\end{align*}
In the equation, we use $|\cdot|$ to denote the absolute value and $\prob(\hat{Y}, Y)$ should be interpreted as the $K \times K$ probability matrix corresponding to the joint distribution of $\hat{Y}$ and $Y$. 

For instance-independent label noise, e.g. symmetric noise, it is possible to show \parencite{Xu2019} that 
\begin{align*}
    \loss(\model(X), \Noisyy) = \loss(\model(X), Y) + C
\end{align*}
for a constant $C$. Following this, it can be concluded that \cref{def:robustness_condition} is fulfilled. 

Next, we derive the asymptotic (risk) minimisers of the information-theoretic loss function. Assume $\yspace =  \{1, 2\}$, then
\begin{align*}
    \loss(\model(X), Y)  = &  - \log \; |\prob(\hat{Y}=1, Y=1)\prob(\hat{Y}=2, Y=2)  - \prob(\hat{Y}=1, Y=2)\prob(\hat{Y}=2, Y=1)|.
\end{align*}
The loss function is minimised when the absolute value of the determinant is maximised. To determine what this means for the model $\model$, we factorise $\prob(\hat{Y}, Y)$ according to  $\prob(\hat{Y}, Y) = \prob(\hat{Y} \mid Y) \prob(Y)$ and use $\prob(\hat{Y}=2 \mid Y=i) = 1 - \prob(\hat{Y}=1 \mid Y=i)$ to obtain
\begin{align*}
    \loss(\model(X), Y) =& 
    - \log \;  |\prob(Y=1)(1-\prob(Y=1))(\prob(\hat{Y}=1 \mid Y=1)- \prob(\hat{Y}=1 \mid Y=2))| \\ =& - \log \;  |\prob(Y=1)(1-\prob(Y=1))(\expectation_{\hat{Y} \mid Y=1}[\mathbb{I}_{\hat{Y} = 1}]- \expectation_{\hat{Y} \mid Y=2}[\mathbb{I}_{\hat{Y} = 1}])|.
\end{align*}
Using the tower property of conditional expectation, we get
\begin{align*}
    \loss(\model(X), Y) =&
     - \log \;  |\prob(Y=1)(1-\prob(Y=1))(\expectation_{X \mid Y=1}[\expectation_{\hat{Y} \mid X, Y=1}[\mathbb{I}_{\hat{Y} = 1}]]- \expectation_{X \mid Y=2}[\expectation_{\hat{Y} \mid X, Y=2}[\mathbb{I}_{\hat{Y} = 1}]])| \\ =&
     - \log \;  |\prob(Y=1)(1-\prob(Y=1))(\expectation_{X \mid Y=1}[\model_1(X)]- \expectation_{X \mid Y=2}[\model_1(X)])|,
\end{align*}
where the last equality follows as $\hat{Y}$ is independent of $Y$ given $X$. Next, we use Bayes' theorem to rewrite the expression further
\begin{align*}
    \loss(\model(X), Y) =&
     - \log \;  |(1-\prob(Y=1)) \expectation_X[\model_1(X)\prob(Y=1\mid X)] - \prob(Y=1)\expectation_X[\model_1(X)(1-\prob(Y=1|X))]| \\ =&
     - \log \;  |\expectation_X[\model_1(X)(\prob(Y=1|X)-\prob(Y=1))]|.
\end{align*}
As $0 \leq \model_1(X) \leq 1$, the loss is minimised if $\model_1(x) = 1$ (or, alternatively, $\model_1(x) = 0$) for all $x \in \xspace$ for which $\prob(Y=1\mid X=x)-\prob(Y=1) \geq 0$ and $\model_1(x) = 0$ ($\model_1(x) = 1$) for all $x\in \xspace$ with $\prob(Y=1\mid X=x)-\prob(Y=1) < 0$. Hence, the minima are found at $\model_1^*(x) = \mathbb{I}_{\prob(Y=1\mid X=x) \geq \prob(Y=1)}$ and $\model_1^*(x) = \mathbb{I}_{\prob(Y=1\mid X=x) < \prob(Y=1)}$. Notice that for balanced classes, i.e. $\prob(Y=1)=\prob(Y=2)=1/2$, the first minimiser is the same as the risk minimisers for e.g. MAE and Sigmoid loss. Evidently, the information-theoretic loss function does not recover $\truemodel$. In addition, just as with the minima of symmetric loss functions (see proof of \cref{prop:robustness_calibrated}), the asymptotic minima of the information-theoretic loss function are not calibrated in general.

\printbibliography 
\end{refsection}

\end{document}